\documentclass[letterpaper]{article} % DO NOT CHANGE THIS
\usepackage{aaai25}  % DO NOT CHANGE THIS
\usepackage{times}  % DO NOT CHANGE THIS
\usepackage{helvet}  % DO NOT CHANGE THIS
\usepackage{courier}  % DO NOT CHANGE THIS
\usepackage[hyphens]{url}  % DO NOT CHANGE THIS
\usepackage{graphicx} % DO NOT CHANGE THIS
\usepackage{bm}
\usepackage{natbib}  % DO NOT CHANGE THIS AND DO NOT ADD ANY OPTIONS TO IT
\usepackage{caption} % DO NOT CHANGE THIS AND DO NOT ADD ANY OPTIONS TO IT
\usepackage{algorithm}
\usepackage[noend]{algpseudocode}
\usepackage[switch]{lineno}
\usepackage{subcaption}
\usepackage{amsthm} 
\usepackage{placeins}
\usepackage{xspace}
\usepackage{xcolor}
\usepackage{amsmath}
\usepackage{amssymb}
\newtheorem{definition}{Definition}
\newtheorem{example}{Example}
\newtheorem{theorem}{Theorem}[section]
\newtheorem{lemma}[theorem]{Lemma}
\newtheorem{problem}{Problem}
\newtheorem{property}{Property}

\usepackage{newfloat}
\usepackage{listings}
\DeclareCaptionStyle{ruled}{labelfont=normalfont,labelsep=colon,strut=off} % DO NOT CHANGE THIS
\floatstyle{ruled}
\newfloat{listing}{tb}{lst}{}
\floatname{listing}{Listing}
\title{A Preprocessing Framework for Efficient Approximate Bi-Objective Shortest-Path Computation in the Presence of Correlated Objectives}
\author{
    Yaron Halle\textsuperscript{\rm 1},
    Ariel Felner\textsuperscript{\rm 2},
    Sven Koenig\textsuperscript{\rm 3},
    Oren Salzman\textsuperscript{\rm 1}
}
\affiliations{
    \textsuperscript{\rm 1}Technion - Israel Institute of Technology\\
    \textsuperscript{\rm 2}Ben-Gurion University\\
    \textsuperscript{\rm 3}University of California, Irvine
    
    yaron.halle@campus.technion.ac.il, felner@bgu.ac.il, sven.koenig@uci.edu, osalzman@cs.technion.ac.il
}

\usepackage{bibentry}
\begin{document}

\newcommand\algname[1]{\textsf{#1}\xspace}
\newcommand\apex{\algname{A*pex}}
\newcommand\eapex{\algname{GA*pex}}
\newcommand\peapex{\algname{PE-A*pex}}
\newcommand\pegapex{\algname{PE-GA*pex}}
\newcommand\open{\textsc{Open}\xspace}
\newcommand{\ignore}[1]{}
\newcommand{\OS}[1]{{\textcolor{red}{\textbf{OS:} #1}}}
\newcommand{\newText}[1]{{\textcolor{red}{#1}}}
\newcommand{\AF}[1]{{\textcolor{blue}{\textbf{Note To Ariel:} #1}}}
\newcommand{\YH}[1]{{\textcolor{orange}{\textbf{YH:} #1}}}
\newcommand{\vs}{\ensuremath{v_{\text{s}}}\xspace}
\newcommand{\vt}{\ensuremath{v_{\text{t}}}\xspace}
\newcommand\AP{\ensuremath{\mathcal{AP}}\xspace}
\newcommand\ApEd{\ensuremath{\mathcal{AE}}\xspace}
\newcommand\G{\ensuremath{\mathcal{G}}\xspace}
\newcommand\Gtilde{\ensuremath{\tilde{\mathcal{G}}}\xspace}
\newcommand\GtildeC{\ensuremath{\tilde{\mathcal{G}_c}}\xspace}
\newcommand\GtildeR{\ensuremath{\tilde{\mathcal{G}_r}}\xspace}
\newcommand\beps{\ensuremath{\boldsymbol{\varepsilon}}\xspace}
\newcommand\Aset{\ensuremath{\mathbb{A}_{\beps}(\psi,b_i,b_j)}\xspace}
\newcommand\pibr{\ensuremath{\pi^{\texttt{br}}}\xspace}
\newcommand\pitl{\ensuremath{\pi^{\texttt{tl}}}\xspace}
\newcommand\epstl{\ensuremath{\boldsymbol{\varepsilon^{\texttt{\textbf{tl}}}}}\xspace}
\newcommand\epsbr{\ensuremath{\boldsymbol{\varepsilon^{\texttt{\textbf{br}}}}}\xspace}
\newcommand{\apexnode}{apex-path pair\xspace} %changed to pair, change back if you prefer
\newcommand{\apexnodes}{apex-path pairs\xspace} 

\maketitle

\begin{abstract}
    The bi-objective shortest-path (BOSP) problem seeks to find paths between  start and target vertices of a graph while optimizing two conflicting objective functions.
    We consider the BOSP problem in the presence of  \emph{correlated objectives}. Such correlations often occur in real-world settings such as road networks, where optimizing two positively correlated objectives, such as travel time and fuel consumption, is common. BOSP is generally computationally challenging as the size  of the search space is exponential in the number of objective functions and the graph size. 
    Bounded sub-optimal BOSP solvers such as
    \apex alleviate this complexity by \emph{approximating} the Pareto-optimal solution set rather than computing it exactly (given some user-provided approximation factor). As the correlation between objective functions increases, smaller approximation factors are sufficient for collapsing the entire Pareto-optimal set into a single solution. 
    We leverage this insight to propose an efficient algorithm that reduces the search effort in the presence of correlated objectives. 
    Our approach for computing approximations of the entire Pareto-optimal set is inspired by graph-clustering algorithms. It uses a preprocessing phase to identify correlated clusters within a graph and to generate a new graph representation. This allows a natural generalization of \apex to run up to five times faster on DIMACS dataset instances, a standard benchmark in the field. To the best of our knowledge, this is the first algorithm proposed that efficiently and effectively exploits correlations in the context of bi-objective search while providing theoretical guarantees on solution quality.
\end{abstract}

\section{Introduction and Related Work}
%In the Bi-Objective Shortest-Path (BOSP) problem~\cite{vincke1976new,current1993multiobjective,ulungu1994multi,skriver2000classification,tarapata2007selected,climaco2012multicriteria}, 

In the bi-objective shortest-path (BOSP) problem~\cite{ulungu1994multi,skriver2000classification,tarapata2007selected}, we are given a directed graph where each edge is associated with two cost components. A path $\pi$ dominates a path $\pi'$ iff each cost component of $\pi$ is no larger than the corresponding component of $\pi'$, and at least one component is strictly smaller. The goal is to compute the Pareto-optimal set of paths from a start vertex \vs to a target vertex \vt, i.e., all undominated paths connecting \vs to \vt.

BOSP models various real-world scenarios, such as minimizing both distance and tolls in road networks or finding short paths that ensure sufficient coverage in robotic inspection tasks~\cite{FuKSA23}.

A long line of research has extended the classical \algname{A*} search algorithm to the multi-objective setting. \algname{MOA*}~\cite{stewart1991multiobjective} and its successors \citep{mandow2008multiobjective,pulido2015dimensionality} propose various techniques for improving performance, which were recently generalized into a unified framework~\cite{ren2025emoa}. 

BOSP is more challenging than single-objective search as it involves simultaneously optimizing two, often conflicting, objectives. The size of the Pareto-optimal solution set can be exponential in the size of the search space, making it computationally challenging to compute precisely \cite{ehrgott2005multicriteria,breugem2017analysis}. 

While exact algorithms have been proposed for BOSP \cite{skyler2022bounded,hernandez2023simple}, we are often interested in \emph{approximating} the Pareto-optimal solution set (see, e.g.,~\cite{perny2008near,tsaggouris2009multiobjective,goldin2021approximate}).

%Interestingly, in many real-world  BOSP settings, the two objectives are correlated, either positively (objectives are aligned) or negatively (objectives are conflicting). 

\begin{figure*}[t]
    \centering
    \includegraphics[scale=0.88]{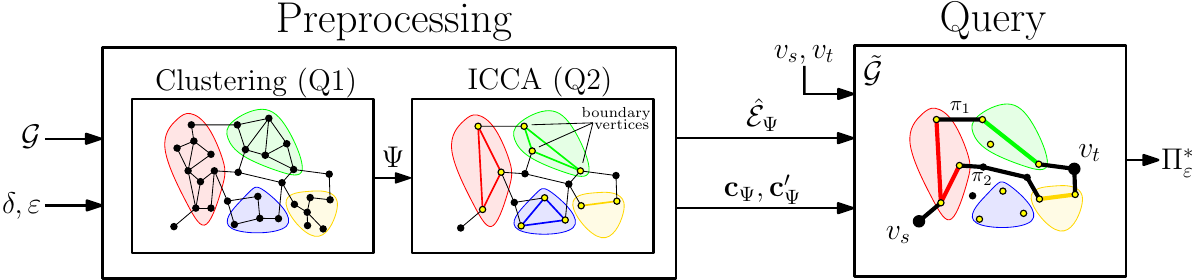}
    %\vspace{-2mm}
    \caption{ 
        Illustration of the proposed algorithmic framework. An input graph \G undergoes a preprocessing phase where regions with similar correlations are grouped into correlated clusters ($\Psi$) using an aggregation threshold~($\delta$).  
        Subsequently, an Internal Cluster Cost Approximation (ICCA) process is performed for every correlated cluster to generate super-edges ($\hat{\mathcal{E}}_\Psi$) along with two cost functions ($\mathbf{c}_\Psi,\mathbf{c}'_\Psi$) for efficiently approximating (using a user-provided approximation factor \beps) the Pareto frontier of paths connecting the cluster's boundary vertices.      
        These are then used to construct a query graph \Gtilde for efficiently computing~$\Pi^*_{\beps}$, an approximation of the Pareto-optimal set for a bi-objective shortest-path query from \vs to \vt.        
    }
    \label{fig:algorithmic_framework}
    \vspace{-4mm}
\end{figure*}

We follow this line of work of approximating the Pareto-optimal solution set but focus on settings in which the objectives are positively correlated.
Such correlations often exist in many real-world settings.
For instance, in road networks, one may consider optimizing two positively correlated objectives such as travel time and fuel consumption. 
In the extreme case where there is perfect positive correlation between two objectives, the problem essentially collapses to a single-objective shortest-path problem and the Pareto-optimal set contains exactly one solution.
Importantly, when the two objectives are strongly (though not perfectly) positively correlated, the Pareto-optimal set may contain many solutions but they are typically very similar in terms of their costs~\cite{brumbaugh1989empirical,mote1991parametric}.  
Consequently, they can all be approximated by a single solution using a small value of approximation factor. 

Surprisingly, despite the relevance to real-world applications and the potential to exploit correlation, this problem has largely been overlooked by the research community.
Unfortunately, the correlation between the objectives can follow complex, non-uniform patterns that are challenging to exploit. Different regions of the graph can exhibit different levels of correlation, 
and their spatial distribution can significantly influence how large the approximation factor is required to be in order to approximate the entire Pareto frontier  by a single solution.

Notable exceptions include empirical studies showing that, 
in the bi-objective setting,  the cardinality of the Pareto-optimal set typically  decreases as the positive correlation increases~\cite{brumbaugh1989empirical, mote1991parametric} 
and that, 
in the more general multi-objective setting, the size of the Pareto-optimal set increases significantly for negative (conflicting) correlations \cite{verel2013structure}.
Recently, Salzman et al.~(\citeyear{salzman2023heuristic}) identified the potential of leveraging correlations to accelerate bi- and multi-objective search algorithms.
To the best of our knowledge, our work is the first one to propose a practical, systematic approach to address this opportunity.

Our approach, summarized in Fig.~\ref{fig:algorithmic_framework}, consists of a preprocessing phase and a query phase. In the preprocessing phase, regions, or clusters, of the bi-objective graph \G with strong correlation between objectives are identified. 
The set of paths within each cluster that connect vertices that lie on the cluster's boundary is efficiently approximated.
In the query phase, a new graph is constructed that allows the search to avoid generating \ignore{search} nodes within these clusters.

Key to our efficiency is a natural generalization of \apex, a state-of-the-art approximate multi-objective shortest-path algorithm~\cite{zhang2022pex}. 
\apex was chosen following its successful application in a variety of bi- and multi-objective settings~\cite{HSFKK24,HSFUK24,ZhangSFKSUK23}.
\ignore{The combination of our framework and our generalization of \apex allows dramatically speeding up the search in the presence of correlated objectives.}
We demonstrate the efficacy of our approach on the commonly-used DIMACS dataset, yielding runtime improvements of up to~ $\times5$ compared to running \apex on the original graph.

\section{Notation and Problem Formulation}
\label{sec:pdef}
We follow standard notation in BOSP~\cite{salzman2023heuristic}:
Boldface indicates vectors, lower-case and upper-case symbols indicate elements and sets, respectively. $p_i$ is used to denote the $i$’th component of vector $\textbf{p}$. 

Let $\textbf{p},\textbf{q}$ be two-dimensional vectors. We define their element-wise summation and multiplication as $\textbf{p}+\textbf{q}$ and $\textbf{p} \cdot \textbf{q}$, respectively.
Similarly, we define their element-wise relational operator as $\bold{p} \leq \bold{q}$.
We say that $\textbf{p}$ \emph{dominates}~$\textbf{q}$ and denote this as $\textbf{p} \prec \textbf{q}$ iff $p_1 \leq q_1$ and $p_2 < q_2$ or if $p_1 < q_1$ and $p_2 \leq q_2$.
When $\textbf{p}$ does not dominate~$\textbf{q}$, we write $\textbf{p} \nprec \textbf{q}$. For $\textbf{p} \neq \textbf{q}$, if $\textbf{p} \nprec \textbf{q}$ and $\textbf{q} \nprec \textbf{p}$, we say that $\textbf{p}$ and $\textbf{q}$ are \emph{mutually undominated}. 
Given a set $\textbf{X}$ of two-dimensional distinct vectors, we say that~$\textbf{X}$ is a \emph{mutually undominated set} if all pairs of vectors in~$\textbf{X}$ are mutually undominated. 
%Note that the minimal undominated set is not necessarily unique. 

Let \beps be another two-dimensional vector such that~ $\varepsilon_1, \varepsilon_2 \geq~0$. We say that~$\textbf{p}$ \emph{\beps-dominates} $\textbf{q}$ and denote this as~$\textbf{p} \preceq_{\beps} \textbf{q}$ iff $\forall i: \ p_i \leq (1+\varepsilon_i) \cdot q_i$.

A bi-objective search graph is a tuple $\mathcal{G}=(\mathcal{V},\mathcal{E},\textbf{c})$, where $\mathcal{V}$ is the finite set of vertices, $\mathcal{E} \subseteq \mathcal{V} \times \mathcal{V}$ is the finite set of edges, and $\textbf{c} : \mathcal{E} \rightarrow \mathbb{R}^2_{\geq 0}$ is a \emph{cost function} that associates a two-dimensional non-negative cost vector with each edge. A \emph{path}~$\pi$ from $v_1$ to $v_n$ is a sequence of vertices $v_1,v_2,\ldots,v_n$ such that $(v_i,v_{i+1}) \in \mathcal{E}$ for all~$i\in~\{1,\ldots,n-1\}$. We define the \emph{cost} of a path $\pi=v_1,\ldots,v_n$ as $\textbf{c}(\pi)=\sum_{i=1}^{n-1}\textbf{c}(v_i,v_{i+1})$. Finally, we say that $\pi$ \emph{dominates} $\pi'$ and denote this as $\pi \prec \pi'$ iff $\textbf{c}(\pi) \prec \textbf{c}(\pi')$.

Given a bi-objective search graph $\G = (\mathcal{V},\mathcal{E},\textbf{c})$ and two vertices $u,v \in \mathcal{V}$,
we denote a minimal set of mutually undominated paths from~$u$  to~$v$ in \G  by~$\Pi^*(u,v)$.
Similarly, given an approximation factor $\beps$, let~$\Pi^*_{\beps}(u,v)$ be a set of paths such that every path in~$\Pi^*(u,v)$ is $\beps$-dominated by a path in $\Pi^*_{\beps}(u,v)$. We denote this as $\Pi^*_{\beps} \preceq_{\beps} \Pi^*$.

For the specific case of a query for start and target vertices $\vs,\vt \in \mathcal{V}$ we set
$\Pi^*:=\Pi^*(\vs,\vt)$ and $\Pi^*_{\beps}:=\Pi_{\beps}^*(\vs,\vt)$
and refer to them as 
a \emph{Pareto-optimal solution set} and 
an $\beps$-\emph{approximate Pareto-optimal solution set}, respectively.
We call the costs of paths in $\Pi^*$ the \emph{Pareto frontier}.

We call the problems of computing~$\Pi^*$ and  $\Pi^*_{\beps}$ the \emph{bi-objective shortest-path} problem and \emph{bi-objective approximate shortest-path} problem, respectively.
In our work, we are interested in a slight variation of these problems where we wish to answer \emph{multiple} bi-objective approximate shortest-path problems given a preprocessing stage.
This is formalized in the following definition.

\begin{problem}
    Let 
    $\G = (\mathcal{V},\mathcal{E},\textbf{c})$ be 
    a bi-objective search graph 
    and
    $\beps \in \mathbb{R}^2_{\geq 0}$ 
    a user-provided approximation factor.
    Our problem calls for preprocessing the inputs \G and $\beps$
    such that, given a query in the form $\vs,\vt \in \mathcal{V}$, we can efficiently compute $\Pi^*_{\beps}(\vs,\vt)$.
\end{problem}

\section{Algorithmic Background}
This section provides the necessary algorithmic background for our  framework. We begin in Sec. \ref{sec:prelim_correlation_in_bosp} by defining correlations between objectives in the context of graph search. Then, in Sec. \ref{sec:apex}, we overview \apex.

\subsection{Correlation in BOSP}
\label{sec:prelim_correlation_in_bosp}
Given two vectors $\mathbf{X}$ and~$\mathbf{Y}$, the correlation coefficient $\rho_{\mathbf{X},\mathbf{Y}}$ quantifies the strength of their \emph{linear} relationship~\cite{pearson1895vii}, ranging from $-1$ (perfect negative correlation) to $1$ (perfect positive correlation). As $|\rho_{\mathbf{X},\mathbf{Y}}|$ approaches~$1$, $\mathbf{X}$ and~$\mathbf{Y}$ become more linearly dependent, meaning that $\mathbf{Y}$ can be closely approximated by a linear equation of the form $\mathbf{Y} = a\mathbf{X} + b$.

\begin{definition}[correlation between objectives]
    Let $E\subseteq \mathcal{E}$ be a set of edges such that each edge $e \in E$ is associated with cost  $\textbf{c}(e)=\left(c_1(e),c_2(e)\right)$. Let $\boldsymbol{C_1}(E)$ ($\boldsymbol{C_2}(E)$ resp.) be a vector of size $\vert E \vert$ comprised of all $c_1(e)$ ($c_2(e)$ resp.) values of every edge $e \in E$. 
    We define the correlation between objectives of the set $E$  as the correlation between vectors $\boldsymbol{C_1}(E)$ and~$\boldsymbol{C_2}(E)$ and denote it as $\rho_E:=\rho_{\boldsymbol{C_1}(E),\boldsymbol{C_2}(E)}$. 
\end{definition}

Correlation between objectives is a common phenomenon. For instance, in the \textsc{9th DIMACS Implementation Challenge: Shortest Path} dataset\footnote{http://www.diag.uniroma1.it/challenge9/download.shtml.}, a widely used benchmark in the BOSP research community, a strong correlation between objectives can be observed. Each instance in this dataset represents a road network graph from various areas in the USA and includes two objectives: driving time and travel distance. The correlation between objectives for an entire graph is roughly~$\rho_{\mathcal{E}}~\approx~0.99$ for most DIMACS instances.

For brevity, when we mention a strong correlation, we specifically mean a strong \emph{positive} correlation.

\subsection{Approximating $\Pi^*$ using \apex}
\label{sec:apex}
In this section, we review  \apex \cite{zhang2022pex}, a state-of-the-art multi-objective best-first search algorithm for approximating the Pareto-optimal solution set. 

\ignore{\subsubsection{Overview of (vanilla) \apex}}
The efficiency of \apex stems from how it represents subsets of the Pareto frontier using one representative path together with a lower bound on the rest of the paths in the subset. 
Specifically, an \emph{apex-path pair}~$\AP=\langle \bold{A},\pi \rangle$ consists of a cost vector $\bold{A}$, called the \emph{apex}, and a path $\pi$, called the \emph{representative path}.
Conceptually, an apex-path pair represents a set of paths, that share the same start and final vertices, with its apex serving as the element-wise minimum of their cost vectors. 
We define 
the $\bold{g}$-value of $\AP$ as $\bold{g}(\AP)=\bold{A}$ 
and~$v(\AP)$ to be the last vertex of~$\pi$. The $\bold{f}$-value of $\AP$ is $\bold{f}(\AP)=\bold{g}(\AP)+\bold{h}(v(\AP))$. 
An apex-path pair \AP is said to be~\emph{\beps-bounded} iff $\boldsymbol{c}(\pi)+~\bold{h}(v(\AP)) \preceq_\varepsilon \bold{f}(\AP)$. 

\apex maintains a priority queue \open, using \beps-bounded apex-path pairs as search nodes.
At each iteration, \apex extracts from \open the node $\AP=\langle \bold{A},\pi \rangle$ with the smallest $\bold{f}$-value. 
If the representative path~$\pi$ has no chance to be part of the approximate solution set due to \beps-domination checks, the node is discarded.
If it does and $v(\AP)=\vt$, the node is added to the solution set.
If none of the above holds, $\AP$ is \emph{expanded} using each outgoing edge of $v(\AP)$ to generate its successor apex-path pair~$\AP'=\langle \bold{A}',\pi'\rangle$. Formally, given an outgoing edge~$e=~(v(\AP),v(\AP'))$, $\AP'$ is obtained by setting~$\bold{A}'$ to the element-wise sum of $\bold{A}$ and $\bold{c}(e)$, and setting $\bold{c}(\pi')$ to the element-wise sum of $\bold{c}(\pi)$ and $\bold{c}(e)$ (see Fig.~\ref{fig:apex_expand_original}). Since $\AP$ is~\beps-bounded, $\AP'$ is also \beps-bounded.

\begin{figure}
    \centering
    \begin{subfigure}[t]{0.14\textwidth}  % Reduce width slightly if necessary
        \centering        
        \includegraphics[width=\textwidth]{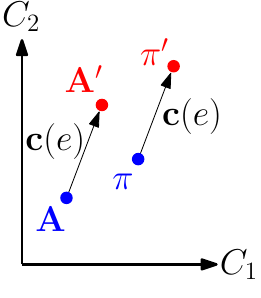}
        \caption{}         
        \label{fig:apex_expand_original}
    \end{subfigure}
    \hfill
    \begin{subfigure}[t]{0.14\textwidth}
        \centering
        \includegraphics[width=\textwidth]{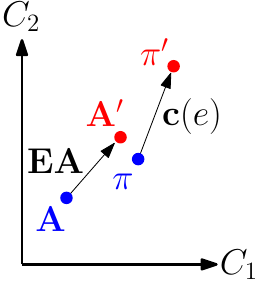}
        \caption{}         
        \label{fig:apex_expand_generalized}
    \end{subfigure}
    \vspace{-2.5mm}
    \caption{\textbf{(a)} 
            \apex expanding an apex-path pair $\AP=~\langle \bold{A},\pi \rangle$ by edge $e$ to obtain $\AP'=\langle \bold{A'},\pi' \rangle$.
            \newline\textbf{(b)} \eapex expanding an apex-path pair $\AP=\langle \bold{A},\pi \rangle$ by an apex-edge pair $\ApEd=\langle \bold{EA}, {e} \rangle$ to obtain $\AP'=\langle \bold{A'},\pi' \rangle$.}            
    \label{fig:apex_expand_combined}
    \vspace{-3.5mm}
\end{figure}

When \apex adds an apex-path pair~$\AP$ to \open, it first tries to \emph{merge} $\AP$ with all other apex-path pairs in \open with the same $v(\AP)$ to reduce the number of search nodes.
When merging two apex-path pairs, the new apex is the element-wise minimum of the apexes of the two apex-path pairs, and the new representative path is either one of the original representative paths (see Fig.~\ref{fig:apex_merge}). If the resulting apex-path pair is \beps-bounded, the merged apex-path pair is used instead of the two original apex-path pairs. 

\begin{figure}[t]
    \centering
    \includegraphics[scale=0.6]{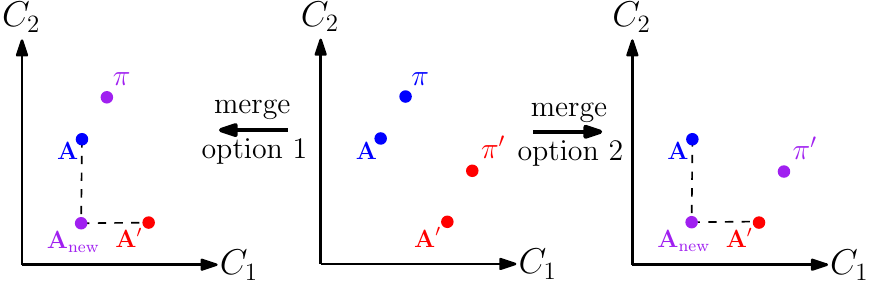}
    \caption{
        \apex merge operation. The new apex-path pair's representative path can be either $\pi$ or $\pi'$.
        }
    \label{fig:apex_merge}
    \vspace{-3.5mm}
\end{figure}

When \open becomes empty, \apex terminates and returns the representative paths of all apex-path pairs in the solution set as an \beps-approximate Pareto-optimal solution set.

\ignore{
\subsubsection{Extension of \apex to apex-edge pairs }
As we will see, it will be useful to apply the notion of a representative path and an associated apex to arbitrary pairs of vertices and not only from paths that start at $\vs$.  Thus, we introduce a natural generalization of edges which we call \emph{apex-edge pairs} and show how this can be seamlessly integrated into \apex by generalizing the way it extends edges.

Specifically, given vertices $u,v$, an apex-edge pair 
$\ApEd = \langle \bold{EA}, {e} \rangle$ 
consists of 
a \emph{representative edge} ${e}$ corresponding to a path connecting $u$ and $v$ 
and
an \emph{edge apex} $\bold{EA}$ which serves as a lower bound to a region of  the Pareto-optimal frontier of $\Pi^*(u,v)$. 
Similar to apex-path pairs, we say that an apex-edge pair \ApEd is $\beps$-bounded if $\bold{c}({e})\preceq_{\beps} \bold{EA}$.

We now generalize the \emph{extend} operation of \apex to account for apex-edge pairs.
Let \AP$=\langle \bold{A}, \pi \rangle$ be an $\beps$-bounded apex-path pair and let $\ApEd=\langle \bold{EA}, {e} \rangle$ be an outgoing $\beps$-bounded apex-edge pair connecting $v(\pi)$ to some vertex $v'$. 
Extending \AP by \ApEd corresponds to a new apex-path pair~$\AP'=~\langle \bold{A'}, \pi' \rangle$ where
(i)~$\pi':=\pi \cdot v$ with $\cdot$ denoting appending a vertex to a path,
(ii)~$\bold{c}(\pi'):= \bold{c}(\pi) + \bold{c}({e})$
and
(iii)~$\bold{A}': =\bold{A} + \bold{EA}$.

\textbf{Note.} We can consider an edge $e$ as degenerate or trivial apex-edge pair $\langle \bold{A}, e \rangle$ where the representative edge $e$ and the edge apex $\bold{A}$ are identical. Moreover, extending a path pair by such a trivial apex-edge pair corresponds to the way \apex  extends operation.

Importantly, using apex-edge pairs and the generalized extend operation preserve all the favorable properties of \apex. This is formalized in the following Lemma.

\begin{lemma}
\label{lemma:eps-bound}
Let $\AP=\langle \bold{A}, \pi \rangle$ be an $\beps$-bounded apex-path pair and 
let $\ApEd=\langle \bold{EA},{e} \rangle$ be an outgoing $\beps$-bounded apex-edge pair connecting $v(\pi)$ to some vertex $v'$. 
If~$\AP'=~\langle \bold{A'},\pi' \rangle$ is the apex-path pair constructed by extending \AP by \ApEd, 
then~$\AP'$ is $\beps$-bounded.
\end{lemma}

\begin{proof}
Since \AP and \ApEd are $\beps$-bounded,
$\bold{c}(\pi) \leq (1+\beps) \cdot \bold{A}$
and
$\bold{c}({e}) \leq (1+\beps) \cdot \bold{EA}$.
Combining the two inequalities:
\[
\underbrace{\bold{c}(\pi)+\bold{c}({e})}_{\bold{c}(\pi')} 
\leq 
(1+\boldsymbol{\varepsilon}) \cdot \underbrace{(\bold{A}+\bold{EA})}_{\bold{A}'}.
\]
Thus, by definition $\bold{c}(\pi')\preceq_{\beps} \bold{A}'$ and $\AP'$ is $\beps$-bounded.
\end{proof}
}

\section{Generalized \apex (\eapex)}
As we will see, it will be useful to apply the notion of a representative path and an associated apex to \emph{arbitrary} paths and not only to paths starting at $\vs$, for introducing \emph{super-edges}.  Thus, we introduce a natural generalization of edges which we call \emph{apex-edge pairs} and show how apex-edge pairs can be seamlessly integrated into \apex by generalizing the way \apex expands apex-path pairs.

\subsection{Apex-Edge Pair Description}
Given vertices $u,v$, an apex-edge pair 
$\ApEd = \langle \bold{EA}, {e} \rangle$ 
consists of 
a \emph{representative edge} ${e}$ corresponding to a path connecting $u$ and $v$ 
and
an \emph{edge apex} $\bold{EA}$, which serves as a lower bound to a subset of  the Pareto-optimal frontier of~$\Pi^*(u,v)$. 
Similar to apex-path pairs, we say that an apex-edge pair \ApEd is $\beps$-bounded iff $\bold{c}({e})\preceq_{\beps} \bold{EA}$.

We now generalize the \emph{expand} operation of \apex to account for apex-edge pairs.
Let \AP$=\langle \bold{A}, \pi \rangle$ be an $\beps$-bounded apex-path pair, and let $\ApEd=\langle \bold{EA}, {e} \rangle$ be an $\beps$-bounded apex-edge pair, where $e$ connects $v(\AP)$ to some vertex $v'$. 
Expanding  \AP by \ApEd corresponds to a new apex-path pair~$\AP'=~\langle \bold{A'}, \pi' \rangle$ where
(i)~$\pi':=\pi \cdot v$ with $\cdot$ denoting appending a vertex to a path,
(ii)~$\bold{c}(\pi'):= \bold{c}(\pi) + \bold{c}({e})$,
and
(iii)~$\bold{A}': =\bold{A} + \bold{EA}$ (see Fig.~\ref{fig:apex_expand_generalized}).

\textbf{Note.} 
Given an edge $e$, we define the corresponding 
\emph{trivial apex-edge pair} $\langle \bold{EA}, e \rangle$  such that the edge apex $\bold{EA}$ equals~$\bold{c}({e})$.
Now, replacing every edge in a graph with the corresponding trivial apex-edge pair, the result of the expansion operation just described is identical to how \apex expands apex-path pairs using edges.
Similarly, running \eapex (which we will describe shortly) when using only trivial apex-edge pairs is identical to how \apex expands apex-path pairs using the corresponding edge.

\subsection{Generalized  \apex}
\label{sec:generalized_apex}
Formally, let
$\mathcal{V}$ and $\mathcal{E}$ be a vertex set and edge set, respectively, and let
$\mathbf{c}, \mathbf{c'} : \mathcal{E} \rightarrow \mathbb{R}^2_{\geq 0}$ be two bi-objective cost  functions  over the edge set $\mathcal{E}$ such that $\mathbf{c'}(e) \preceq \mathbf{c}(e)$.
We define $\hat{\mathcal{G}} := (\mathcal{V}, \mathcal{E},\mathbf{c}, \mathbf{c'})$ and refer to it as a \emph{generalized graph}.
For each edge $e$ in the generalized graph, we define the \emph{corresponding} apex-edge pair $\ApEd = \langle \bold{EA}, {e} \rangle$ such that 
$\bold{EA}$ and the cost of $e$ are $\mathbf{c'}(e)$ and $\mathbf{c}(e)$, respectively.

In contrast to \apex which runs on graphs, 
\eapex runs on generalized graphs. However, the two algorithms only differ in how they expand apex-path pairs.
Specifically, \apex running on graph $\mathcal{G}=  (\mathcal{V}, \mathcal{E},\mathbf{c})$
is identical to \eapex running on graph $\hat{\mathcal{G}} = (\mathcal{V}, \mathcal{E},\mathbf{c}, \mathbf{c'})$
except that, when 
\apex expands an apex-path pair using edge~$e$, 
\eapex expands the apex-path pair using $e$'s corresponding apex-edge pair.

\begin{lemma}
\label{lemma:eps-bound}
Let $\AP=\langle \bold{A}, \pi \rangle$ be an $\beps$-bounded apex-path pair, and 
let $\ApEd=\langle \bold{EA},{e} \rangle$ be an $\beps$-bounded apex-edge pair whose representative edge $e$ connects $v(\AP)$ to some vertex~$v'$. 
If~$\AP'=~\langle \bold{A'},\pi' \rangle$ is the apex-path pair constructed by expanding \AP by \ApEd, 
then~$\AP'$ is $\beps$-bounded.
\end{lemma}
{
The proof is straightforward, as \beps-boundedness is preserved under component-wise addition of cost vectors. The full proof is provided in Appendix~\ref{apndx_lemma_eps_bound}.
}

\ignore{
\begin{proof}
    $\AP=\langle \bold{A}, \pi \rangle$ is \beps-bounded, thus:
    \begin{equation}
        \label{eq:ap_eps_bounded}
        \mathbf{c}(\pi) \leq (1+\beps) \cdot \mathbf{A}
    \end{equation}

    $\ApEd=\langle \bold{EA},{e} \rangle$ is \beps-bounded, thus:
    \begin{equation}
        \label{eq:ae_eps_bounded}
        \mathbf{c}(e) \leq (1+\beps) \cdot \mathbf{EA}
    \end{equation}

    The cost of path $\pi'$ is:
    \begin{eqnarray}
        \begin{split}
        \mathbf{c}(\pi') &= \mathbf{c}(\pi) + \mathbf{c}(e) \\
        &\underbrace{\leq}_{\eqref{eq:ap_eps_bounded},\eqref{eq:ae_eps_bounded}} (1+\beps) \cdot \bold{A}+(1+\beps) \cdot \bold{EA} \\
        &=(1+\beps) \cdot (\bold{A} + \bold{EA}) \\
        &=(1+\beps) \cdot \mathbf{A'} 
        \end{split}
    \end{eqnarray}
    
    Thus, by definition, $\AP'$ is \beps-bounded.
\end{proof}
}

\ignore{
\begin{lemma}
\label{lemma:apx-bound}
Let $\mathcal{V}$ and $\mathcal{E}$ be a vertex set and an edge set, respectively, 
and
let $\mathbf{c}, \mathbf{c'} : \mathcal{E} \rightarrow \mathbb{R}^2_{\geq 0}$ be two bi-objective cost  functions  over the edge set $\mathcal{E}$ such that $\forall e \in \mathcal{E}, \mathbf{c'}(e) \leq \mathbf{c}(e)$.
Set $\beps := \underset{e\in \mathcal{E}}{\max} \left(\mathbf{c}(e) / \mathbf{c'}(e) - 1\right)$.
The Pareto-optimal solution set of paths between \vs and \vt in graph $\mathcal{G}_{c}=  (\mathcal{V}, \mathcal{E},\mathbf{c})$ is an $\beps$-approximation of the Pareto-optimal solution set of paths between~\vs and \vt in graph $\mathcal{G}_{c'}=  (\mathcal{V}, \mathcal{E},\mathbf{c'})$.
\end{lemma}
\begin{proof}
    From the definition of $\beps$ we get that $\forall e \in \mathcal{E}$:
    \begin{eqnarray}
    \label{eq:edge_cost_constr}
    \begin{split}
    \beps &\geq \frac{\mathbf{c}(e)}{\mathbf{c'}(e)}-1 \\
    \Rightarrow \mathbf{c}(e) &\leq (1+\beps) \cdot \mathbf{c'}(e).
    \end{split}
    \end{eqnarray}
        
    We can extend this property to any arbitrary path $\pi$:
    \begin{eqnarray}
        \label{eq:path_cost_constr}
        \begin{split}
        \mathbf{c}(\pi)&=\sum_{i=1}^{|\pi|-1}\mathbf{c}\Bigl( (v_i,v_{i+1}) \Bigr) \\
        &\underset{(\ref{eq:edge_cost_constr})}{\leq} (1+\beps)\sum_{i=1}^{|\pi|-1}\mathbf{c'}\Bigl( (v_i,v_{i+1}) \Bigr) \\
        &=(1+\beps) \cdot \mathbf{c'}(\pi).
        \end{split}
    \end{eqnarray}

    Let us denote by $\Pi^*_{\bold{c}}$ ($\Pi^*_{\bold{c'}}$, resp.) the Pareto-optimal solution set $\Pi^*(\vs,\vt)$ in $\mathcal{G}=(\mathcal{V},\mathcal{E},\bold{c})$ ($\mathcal{G'}=(\mathcal{V},\mathcal{E},\bold{c'})$, resp.).   
    Let $\pi' \in \Pi^*_{\bold{c'}}$ and consider the following two cases:
    \begin{enumerate}
        \item[\textbf{C1}] $\pi' \notin \Pi^*_{\bold{c}}$. 
        Thus there must exist another path $\pi \in \Pi^*_{\bold{c}}$ such that 
        \begin{equation}
        \label{eq:dominance_case1a}
            \mathbf{c}(\pi) \leq \mathbf{c}(\pi').
        \end{equation}
        Using Eq.~\eqref{eq:path_cost_constr} we have that
        \begin{equation}
        \label{eq:dominance_case1b}
            \mathbf{c}(\pi') \leq (1+\beps) \cdot \mathbf{c'}(\pi').
        \end{equation}
        Thus, from Eq.~\eqref{eq:dominance_case1a} and~\eqref{eq:dominance_case1b} we have that 
        \begin{equation}
        \label{eq:dominance_case1}
            \mathbf{c}(\pi) \leq (1+\beps) \cdot \mathbf{c'}(\pi').
        \end{equation}
        %%%%OLD text%%%%
%        In this case, $\exists \pi \in \Pi^*_{\bold{c}}$ that satisfies:
 %       \begin{equation}
  %      \label{eq:dominance_case1}
   %         \mathbf{c}(\pi) \preceq_{\beps} \mathbf{c'}(\pi')
    %    \end{equation}
        %
        \item[\textbf{C2}] $\pi' \in \Pi^*_{\bold{c}}$. 
        In this case, following Eq.~\eqref{eq:path_cost_constr}:
        \begin{eqnarray}
        \label{eq:dominance_case2}
        \begin{split}
        \bold{c}(\pi') &\leq (1+\beps) \cdot \bold{c'}(\pi').
        \end{split}
        \end{eqnarray}
    \end{enumerate}
    Combining Eq.~\eqref{eq:dominance_case1} and Eq.~\eqref{eq:dominance_case2} we get that:
    \[
        \forall \pi' \in \Pi^*_{\bold{c'}}, \ \exists \pi \in \Pi^*_{\bold{c}} :
        \mathbf{c}(\pi) \leq (1+\beps) \cdot \mathbf{c'}(\pi').
    \]
    By definition, it implies that $\Pi^*_{\bold{c}}$ \beps-dominates $\Pi^*_{\bold{c'}}$.
\end{proof}
}

\begin{theorem}
\label{thm:gapex}
    Let $\hat{\mathcal{G}} = (\mathcal{V}, \mathcal{E},\mathbf{c}, \mathbf{c'})$    
    be a generalized graph of graph $\mathcal{G}=  (\mathcal{V}, \mathcal{E},\mathbf{c})$. 
    Let~$\vs,\vt \in~\mathcal{V}$ and recall that~$\Pi^*_{\bold{c}}$ denotes the Pareto-optimal set of paths between $v_s$ and~$v_t$ in $\G$.
    Set 
    $\beps := \underset{e\in \mathcal{E}}{\max} \left(\mathbf{c}(e) / \mathbf{c'}(e) - 1\right)$
    and
    let
    $\Pi^*_{\eapex}$ be the output of \eapex on $\hat{\mathcal{G}}$ when using an approximation factor $\beps$.
    Then, $\Pi^*_{\eapex} \preceq_{\beps} \Pi^*_{\bold{c}}$.
    Namely, running  \eapex on the generalized graph with approximation factor $\beps$ yields a Pareto-optimal solution set of paths between \vs and \vt that is an $\beps$-approximation of the Pareto-optimal solution set of paths between $v_s$ and~$v_t$ in \G.
\end{theorem}
{
The proof is almost identical to A*pex optimality proof (Thm. 1 in \cite{zhang2022pex}). For full details we refer the reader to Appendix \ref{apndx_thm_gapex}.}

\ignore{
\begin{proof}
    \newText{New version:}
    
    As described earlier, \apex and \eapex only differ in how they expand apex-path pairs. Therefore, to prove the theorem, we will go through the same theoretical steps in the proof of Thm. 1 in \apex paper \cite{zhang2022pex}. Note that Lemmas 1-3 are trivially correct for \eapex as well. 
    Let us now verify the validity of Lemma 4's proof for \eapex. We mark in blue extensions considering \eapex. The line numbers refer to \apex pseudo-code in \cite{zhang2022pex}. The proof is by induction. The lemma holds for $l=1$ and any solution since apex-path pair $\AP=\langle \mathbf{0},[s_{\text{start}}] \rangle$ gets expanded and has the properties required for Lemma's 4 Case 1. Now assume that the lemma holds for some $l<L$ and any solution. We prove that it then also holds for $l+1$ and this solution.

    Assume that Case 1 holds for $l$ and consider both the apex-path pair \AP mentioned there and its potential child apex-path pair $\AP'$ created on Line 14 for $s'=s_{l+1}$. Apex-path pair $\AP'$ contains state $s_{l+1}$, and its apex weakly dominates the $\mathbf{g}$-value of path $\pi_{l+1}$, which implies that its $\mathbf{f}$-value weakly dominates the $\mathbf{f}$-value of path $\pi_{l+1}$. We distinguish several cases:
    \begin{enumerate}
        \item First, the condition on Line 20 holds for some apex-path pair in the solution set, namely, the truncated $\mathbf{f}$-value of the representative path of this apex-path pair \beps-dominates the truncated $\mathbf{f}$-value of apex-path pair $\AP'$. \apex replaces this apex-path pair with a new apex-path pair $\AP''$ in the solution set on Line 22. Apex-path pair $\AP''$ stays in the solution set but \apex might merge it several (more) times with other apex-path pairs on Line 29 before it terminates. The apex of apex-path pair $\AP''$ weakly dominates the $\mathbf{f}$-value of path $\pi_{l+1}$ (since this apex is the component-wise minimum of the $\mathbf{f}$-value of apex-path pair $\AP'$ and another apex and hence weakly dominates the $\mathbf{f}$-value of apex-path pair $\AP'$, which in turn weakly dominates the $\mathbf{f}$-value of path $\pi_{l+1}$) and merging it with other apex-path pairs does not change this property according to Lemma 3 (in \apex paper). This apex-path also remains \beps-bounded (which is due to the conditions on Lines 20 and 30, Lemma 1 (in \apex paper), the consistent heuristics and \textcolor{blue}{by Lemma \ref{lemma:eps-bound}}), the $\mathbf{f}$-value of its representative path always \beps-dominates the $\mathbf{f}$-value of itself, which equals its apex. Put together, the $\mathbf{f}$-value of its representative path \beps-dominates the $\mathbf{f}$-value of path $\pi_{l+1}$. Thus, the merged apex-path pair satisfies Case 2 for $l+1$.

        \item Second, the condition on Line 24 holds, namely, there exists a truncated $\mathbf{f}$-value in $G^T_{cl}(s(\AP'))$ that weakly dominates the truncated $\mathbf{f}$-value of apex-path pair $\AP'$. Then, an expanded apex-path pair $\AP''$ exists according to Lemma 2 (in \apex paper) that contains state $s_{l+1}$ and whose $\mathbf{f}$-value weakly dominates the $\mathbf{f}$-value of apex-path pair $\AP'$. Thus, its apex weakly dominates the apex of apex-path pair $\AP'$. Thus, apex-path pair $\AP''$ satisfies Case 1 for $l+1$ since the apex of apex-path pair $\AP'$ in turn weakly dominates the $\mathbf{g}$-value of path $\pi_{l+1}$.

        \item Otherwise, \apex executes Line 17 for apex-path pair $\AP'$, where the apex-path pair is inserted into \open, perhaps after having been merged with another apex-pair pair on Line 29. \apex might merge it several (more) times with other apex-path pairs on Line 29 before finally extracting it. Its apex weakly dominates the $\mathbf{g}$-value of path $\pi_{l+1}$ and merging it with other apex-path pairs does not change this property according to Lemma 3 (in \apex paper). Thus, if this apex-path pair is expanded, it satisfies Case 1 for $l+1$. If it is extracted but not expanded, the condition on Line 20 or Line 24 holds, and thus, as we have already proved, Case 1 or Case 2 holds.
    \end{enumerate}
    Assume that Case 2 holds for $l$ and consider the apex-path pair mentioned there. The $\mathbf{f}$-value of the representative path of this apex-path pair \beps-dominates the $\mathbf{f}$-value of path $\pi_l$. Since the heuristic function is consistent, the $\mathbf{f}$-value of path $\pi_l$ in turn weakly dominates the $\mathbf{f}$-value of path $\pi_{l+1}$. Thus, this apex-path pair satisfies Case 2 for $l+1$.

    \paragraph{Proof of Thm. 1 in \apex paper}{} Lemma 4 (in \apex paper) holds for prefix $\pi_L=\pi$ of any solution $\pi$. In case its Case 2 holds, the theorem holds by definition for path $\pi$ since the $\mathbf{f}$-value of solutions (including those of the representative path and path $\pi$) are equal to their costs. In case its Case 1 holds, consider the apex-path pair mentioned there. This apex-path pair contains the goal state, and \apex thus executed Line 11 for it, where the apex-path pair was inserted into the solution set, perhaps after having been merged with another apex-path pair on Line 29. The apex-path pair stays in the solution set but \apex might merge it several (more) times with other apex-path pairs on Line 29 before it terminates. The apex of the apex-path pair weakly dominates the $\mathbf{g}$-value of path $\pi$ according to Lemma 4 (in \apex paper) and merging it with other apex-path pairs does not change this property according to Lemma 3 (in \apex paper). Since the apex-path pair also remains \beps-bounded (which is due to the conditions on Lines 20 and 30, Lemma 1 (in \apex paper), the consistent heuristics and \textcolor{blue}{by Lemma \ref{lemma:eps-bound}}), the $\mathbf{f}$-value of its representative path always \beps-dominates the $\mathbf{f}$-value of itself, which equals its apex. Put together, the $\mathbf{f}$-value of its representative path \beps-dominates the $\mathbf{g}$-value of path $\pi$. Thus, the theorem holds by definition for path $\pi$ since the $\mathbf{g}-$ and $\mathbf{f}$-values of solutions (including those of the representative path and path $\pi$) are equal to their costs.
    
    \ignore{
    \newText{Previous version:}
    We denote $\Pi^*_{\mathbf{c'}}$ as the Pareto-optimal set $\Pi^*(\vs,\vt)$ in graph $\G'=(\mathcal{V},\mathcal{E},\mathbf{c'})$.
    From Lemma \ref{lemma:apx-bound} we get:
    \[
        \forall \pi_{\mathbf{c'}} \in \Pi^*_{\mathbf{c'}},
        \exists \pi_{\eapex} \in \Pi^*_{\eapex}: 
    \]
    \begin{equation}    
    \label{eq:pi_gapex}
        \mathbf{c}(\pi_{\eapex}) \leq (1+\beps) \cdot \mathbf{c'}(\pi_{\mathbf{c'}}).
    \end{equation}
    Let us show that each path $\pi_\mathbf{c} \in \Pi^*$ is weakly dominated by some path $\pi_\mathbf{c'} \in \Pi^*_{\mathbf{c'}}$.
    Similar to Eq.~\eqref{eq:dominance_case1} and Eq.~\eqref{eq:dominance_case2}, we will consider two cases:
    \begin{enumerate}
        \item[\textbf{C1}] $\pi_{\mathbf{c}} \notin \Pi^*_{\bold{c'}}$. In this case, $\exists \pi \in \Pi^*_{\bold{c'}}$ that satisfies:
        \begin{equation}
        \label{eq:dominance_case3}
            \mathbf{c'}(\pi) \leq \mathbf{c}(\pi_{\bold{c}}).
        \end{equation}
        \item[\textbf{C2}] $\pi_{\bold{c}} \in \Pi^*_{\bold{c'}}$. 
        From the definition of a generalized graph, we know that:
        \[
        \forall e \in \mathcal{E}: \mathbf{c'}(e) \leq \mathbf{c}(e).
        \]
        This property trivially extends to paths. Thus:
        \[
            \forall \pi: \mathbf{c'}(\pi) \leq \mathbf{c}(\pi).
        \]
        And specifically for $\pi_{\bold{c}}$:
        \begin{equation}
            \label{eq:dominance_case4}
             \mathbf{c'}(\pi_{\bold{c}}) \leq \mathbf{c}(\pi_{\bold{c}}).
        \end{equation}
    \end{enumerate}    
        Combining Eq.~\eqref{eq:dominance_case3} and Eq.~\eqref{eq:dominance_case4} we get that:
        \begin{equation}
            \label{eq:c_prime_domination}
            \forall \pi_{\bold{c}} \in \Pi^*, \ \exists \pi_{\bold{c'}} \in \Pi^*_{\bold{c'}} :
            \mathbf{c'}(\pi_{\bold{c'}}) \leq \mathbf{c}(\pi_{\bold{c}}).
        \end{equation}
        By definition, it implies that $\Pi^*_{\bold{c'}}$ weakly-dominates $\Pi^*$. Combining Eq.~\eqref{eq:pi_gapex} and Eq.~\eqref{eq:c_prime_domination} we get that:
        \[
            \forall \pi_{\bold{c}} \in \Pi^*, \ \exists \pi_{\eapex} \in \Pi^*_{\eapex} :
        \]
        \[
            \mathbf{c}(\pi_{\eapex}) \leq (1+\beps) \cdot \mathbf{c}(\pi_{\mathbf{c}}).
        \]
    By definition, it implies that $\Pi^*_{\eapex}$ \beps-dominates $\Pi^*$.
    \newText{Note: Lemma (\ref{lemma:eps-bound}) is not needed?!}
    }
\end{proof}
}

\ignore{
We now extend definitions and notations introduced in Sec.~\ref{sec:pdef} to account for apex-edge pairs.
Throughout this section we denote 
$\mathcal{V}$ and $\mathcal{E}$ to be a vertex set and edge set, respectively 
and
$\beps\geq 0$ to be some approximation factor.

\begin{definition}
Let 
$\mathbf{c}, \mathbf{c'} : \mathcal{E} \rightarrow \mathbb{R}^{\geq 0}$ be two bi-objective cost  functions  over the edge set $\mathcal{E}$.
We say that $\mathbf{c'}$ is an $\beps$-lower bound of $\mathbf{c'}$ if 
$\forall e \in \mathcal{E}$ it holds that
(i)~$\mathbf{c'}(e) \leq \mathbf{c}(e)$
and
(ii)~$\mathbf{c}(e) \leq (1 + \beps) \cdot \mathbf{c'}(e) $.
\end{definition}

The generalized version of \apex \OS{WHICH WE CAL???}
receives as input 
a graph $\mathcal{G} = (\mathcal{V}, \mathcal{E})$,
a approximation factor~$\beps\geq 0$
and
$\mathbf{c}, \mathbf{c'} : \mathcal{E} \rightarrow \mathbb{R}^{\geq 0}$ such that 
$\mathbf{c'}$ is an $\beps$-lower bound of $\mathbf{c'}$.
Given a query $\vs,\vt \in \mathcal{V}$ it behaves exactly like \apex with the following change:
while  generating the outgoing instead of taking the neigh

\begin{lemma}
    Let 
    $\G =(\mathcal{V},\mathcal{E})$ to be a graph, 
    $\beps\geq 0$ be some approximation factor
    and
    $\mathbf{c}, \mathbf{c'}$ two bi-objective cost  functions  over $\mathcal{E}$
    s.t. $\mathbf{c'}$ is an $\beps$-lower bound of $\mathbf{c'}$.
    For any~$u,v \in \mathcal{V}$ the Pareto optimal solutions set  of paths connecting $u,v$ in $(\mathcal{V},\mathcal{E}, \mathbf{c})$
    is an $\beps$-approximate Pareto optimal solutions set of the Pareto optimal solutions set  of paths connecting $u,v$ in $(\mathcal{V},\mathcal{E}, \mathbf{c'})$.
\end{lemma}

\begin{definition}
    Let 
    $\G =(\mathcal{V},\mathcal{E})$ to be a graph, 
    $\beps\geq 0$ be some approximation factor
    and
    $\mathbf{c}, \mathbf{c'}$ two bi-objective cost  functions  over $\mathcal{E}$
    s.t. $\mathbf{c'}$ is an $\beps$-lower bound of $\mathbf{c'}$.
    We 
    define~$\hat{\mathcal{E}}({\mathcal{E}}, \mathbf{c}, \mathbf{c'})$, the set of apex-edge pairs 
    corresponding to~$\mathcal{E}$ as follows:
    for each edge $e\in\mathcal{E}$ we add the apex-edge pair $\hat{e}$ to $\hat{\mathcal{E}}$ such that the costs of the apex and edge of $\hat{e}$ are set to $\mathbf{c'}(e)$ and $\mathbf{c}(e)$, respectively. 
    We denote the graph defined over the set of apex-edge pairs corresponding to~$\mathcal{E}$
    by~$\hat{\mathcal{G}} = (\mathcal{V}, \hat{\mathcal{E}})$.
\end{definition}

\begin{lemma}
    For any $u,v \in \mathcal{V}$,
    the output of running \apex on $\hat{\mathcal{G}}$ with approximation factor $\beps$  is an 
    $\beps$-approximate Pareto optimal solutions set of the Pareto optimal solutions set  of paths connecting $u,v$ in $(\mathcal{V},\mathcal{E}, \mathbf{c'})$.
\end{lemma}
}

\ignore{
\YH{Note that super-edges are defined only later in the paper, so I can't really refer specifically to them. I describe a more general idea.}

This was a brief overview of the \apex mechanism. In the \apex original paper, the authors extended an apex-path pair only using regular edges. Here, we introduce a natural generalization of edge extension inspired by the concept of apex-path pairs.

Recall that every apex-path pair is associated with a specific graph vertex corresponding to the end vertex of its representative path. Specifically, an apex-path pair represents a region in the Pareto-optimal front of paths going from the query's \emph{start} vertex to the apex-pair last vertex of the representative path. We generalize this idea for partial paths that start from an \emph{arbitrary} vertex. In analogy to the apex-path pair representation, we define an \emph{apex-edge} pair, denoted as $\ApEd = \langle \bar{\pi}, \bold{EA} \rangle$. Here, the \emph{representative edge} $\hat{e}$ corresponds to a specific sequence of edges along a partial path whose initial vertex is \emph{not} necessarily the query's start vertex. The \emph{edge apex} $\bold{EA}$ serves as a lower bound on the Pareto-optimal cost of all paths that share the same start and end vertices as $\hat{e}$. We say that an apex-edge pair \ApEd is \beps-bounded if $\bold{c}(\hat{e})\preceq_{\boldsymbol{\varepsilon}} \bold{EA}$.

This generalization naturally extends to regular edges: a regular edge is a degenerate case of a partial path consisting of a single edge, where its representative cost and edge apex are trivially identical.

We now generalize the \emph{extend} operation in \apex to incorporate \emph{apex-edge} pairs.

Let \AP$=\langle \pi, \bold{A} \rangle$ be an \beps-bounded apex-path pair, and let $\ApEd=\langle \bar{\pi},\bold{EA} \rangle$ be an outgoing \beps-bounded apex-edge pair originating from $v(\pi)$. The successor apex-path pair~ $\AP'=~\langle \pi', \bold{A'} \rangle$, derived by \emph{extending} \AP with \ApEd
is defined as follows:
\begin{itemize}
    \item Set $\bold{c}(\pi')$ to be the element-wise sum of $\bold{c}(\pi)$ and $\bold{c}(\bar{\pi})$.
    \item Set $\bold{A}'$ to be the element-wise sum of $\bold{A}$ and $\bold{EA}$.
\end{itemize}
We will prove that the generalized \emph{extend} operation in \apex preserves the \beps-bound property of every apex-path pair in \open.

\begin{lemma}
\label{lemma:eps-bound}
The successor apex-path pair $\AP'=\langle \pi', \bold{A'} \rangle$ is \beps-bounded.
\end{lemma}

\begin{proof}
Since \AP is \beps-bounded: $\bold{c}(\pi) \leq (1+\boldsymbol{\varepsilon}) \cdot \bold{A}$, and similarly, since \ApEd is \beps-bounded: $\bold{c}(\bar{\pi}) \leq (1+\boldsymbol{\varepsilon}) \cdot \bold{EA}$. Combining the two inequalities:
\[
\underbrace{\bold{c}(\pi)+\bold{c}(\bar{\pi})}_{\bold{c}(\pi')} 
\leq 
(1+\boldsymbol{\varepsilon}) \cdot \underbrace{(\bold{A}+\bold{EA})}_{\bold{A}'}
\]
So, by definition $\bold{c}(\pi')\preceq_{\boldsymbol{\varepsilon}} \bold{A}'$. Thus, $\AP'$ is \beps-bounded.
\end{proof}

\YH{The theorem of proving correct $\varepsilon$-approximation could not be put here since it makes statements about graph \Gtilde which was not even defined. Thus, the theorem was moved back to the query phase section.}

With a slight abuse of notation, we continue to refer to the adapted \apex, which incorporates the generalized \emph{extend} operation, simply as \apex.

}

\ignore{
\subsection{Computing $\Pi^*_\varepsilon$ for correlated objectives}
\label{sec:computing_Pi_epsilon_for_correlated_objectives}
We expect that in graphs with correlated objectives, relatively small approximation factors will be sufficient for a single solution to $\varepsilon$-dominate the entire Pareto-optimal set. 
To illustrate this, consider a straightforward approach for selecting such a solution. Given a search query between vertices $s$ and $t$, we begin by solving two independent single-objective shortest-path searches - one for each objective. We then determine the conditions under which one of these solutions $\varepsilon$-dominates the entire Pareto-optimal set. Formally, let $\pi_1$ be the optimal solution obtained by optimizing the first objective, yielding a cost of $\boldsymbol{c}(\pi_1)=~(\underline{c}_1,\bar{c}_2)$. Similarly, let $\pi_2$ be the optimal solution obtained by optimizing the second objective, with a cost of $\boldsymbol{c}(\pi_2)=(\bar{c}_1,\underline{c}_2)$. Clearly, $\pi_1$ and $\pi_2$ are mutually undominated, as $\underline{c}_1 \leq \bar{c}_1$ and $\underline{c}_2 \leq \bar{c}_2$. We refer to $\pi_1$ and $\pi_2$ as the \emph{extreme} solutions of $\Pi^*$. Given a user-defined approximation factors vector $\boldsymbol{\varepsilon}=(\varepsilon_1,\varepsilon_2)$, we can approximate $\Pi^*$ using a single solution under the following conditions:
\begin{enumerate}    
    \item Set $\Pi^*_{\varepsilon}\equiv\{\pi_1\}$ if $(\bar{c}_2 - \underline{c}_2)/\underline{c}_2 \leq \varepsilon_2$. Since $\pi_1$ is already optimal for the first objective, no constraint is needed on~$\varepsilon_1$.
    \item Set $\Pi^*_{\varepsilon}\equiv\{\pi_2\}$ if $(\bar{c}_1 - \underline{c}_1)/\underline{c}_1 \leq \varepsilon_1$. Similarly, no constraint is required on $\varepsilon_2$ since $\pi_2$ is optimal for the second objective.
\end{enumerate}

In other words, we choose one of the extreme solutions and verify whether it $\varepsilon$-dominates the other. When a strong correlation exists between objectives, solving only two \emph{single}-objective shortest-path queries may be sufficient to easily obtain~$\Pi^*_\varepsilon$ without explicitly exploring the full \emph{bi}-objective search space, regardless of the number of solutions in the Pareto-optimal set.

Another important factor to consider is the number of edges taken along an optimal trajectory. Empirical tests indicate that as the number of hops in an optimal trajectory increases, the two extreme solutions of the Pareto front are more likely to diverge. Under the proposed approach for generating $\Pi^*_\varepsilon$, this would result in larger approximation factors. \YH{This claim is important as it outlines the tension of cluster delineation - you cannot aggregate as many edges as you want in a single cluster, as more edges extend the Pareto front's width, resulting in exceeding the user's approximation factor threshold. The claim could be theoretically supported by stochastic processes theory which is out of this paper scope. Thus, this claim is presented only empirically.}
}

\section{Algorithmic Approach}
In graph regions with a strong correlation between objectives, while there may be a large number of solutions in the Pareto-optimal solution set, they can typically all be \beps-dominated by a single solution using a small approximation factor.
Following this insight, we propose an algorithmic framework (see Fig.~\ref{fig:algorithmic_framework}) where, in a preprocessing phase (Sec.~\ref{sec:correlation_based_preprocessing}), we identify continuous regions with a strong correlation between objectives, which we call correlated clusters.
To avoid having to run our BOSP search algorithm within each correlated cluster, we then compute a set of apex-edge pairs that allows the approximation of paths that traverse a correlated cluster.
Given a query, these apex-edge pairs are used to construct a new graph, which we call the query graph, and to define a corresponding generalized query graph (Sec.~\ref{sec:query_phase}). 
As we will see, running \eapex on the generalized query graph allows us to compute $\Pi^*_{\beps}$ much faster than running  \apex on the original graph.
The rest of this section formalizes our approach.

\subsection{Correlation-Based Preprocessing}
\label{sec:correlation_based_preprocessing}
We start by introducing several key definitions.

\begin{definition}[conforming edge]
Let $e \in \mathcal{E}$,
$\delta >0$ be some threshold
and 
$\ell$ be some two-dimensional line 
(i.e., $\ell: a  x+ b  y +1 = 0$ for some~$a,b$ s.t.~$a^2 + b^2>0$).
We say that an edge $e$ $\delta$-\emph{conforms} with line $\ell$ iff 
$\texttt{dist}_\perp(\ell,\boldsymbol{c}(e)) \leq \delta$.
Here
    \begin{equation}
        \label{eq:distance_between_edge_and_line}
        \texttt{dist}_\perp(\ell,\boldsymbol{c}(e)) := \frac{\vert ac_1(e)+bc_2(e)+1\vert}{\sqrt{a^2+b^2}}.
    \end{equation}
\end{definition}

\begin{definition}[correlated cluster]
Given a graph~$\mathcal{G}=(\mathcal{V},\mathcal{E})$,
    a $(\delta,\ell)$-correlated cluster of \G is a subgraph~$(V,E)$ of \G, s.t.
    (i)~$V \subseteq \mathcal{V}$ and $E = (V \times V) \cap \mathcal{E}$
    and
    (ii)~$\forall e \in E$, we have that $e$ $\delta$-\emph{conforms} with $\ell$.  
\end{definition}

As we will see, all the $(\delta,\ell)$-correlated clusters we will consider will use the same value of $\delta$. Thus, to simplify exposition and with a slight abuse of notation, we will refer to a $(\delta,\ell)$-correlated cluster $\psi$ simply as a cluster and use $\ell(\psi)$ to obtain the line that all edges of $\psi$ conform with.

\begin{definition}[boundary vertices]
    Let $\psi$ be a correlated cluster of $\mathcal{G}=(\mathcal{V},\mathcal{E})$.
    The set of boundary vertices of $\psi$ in \G, denoted as $B(\psi)$, is defined as
    \[
        B(\psi):=\{u\in V_\psi \ \vert \ \exists v\in \mathcal{V} \setminus V_\psi~s.t.~(u,v)\in \mathcal{E} \ \text{or} \ (v,u)\in \mathcal{E} \}.
    \]    
    In other words, a vertex $u\in V_\psi$ is a boundary vertex iff it has at least one adjacent vertex $v\in \mathcal{V} \setminus V_\psi$.    
\end{definition}

In an $(\ell, \delta)$-correlated cluster of \G,
small values of $\delta$ typically imply that the entire Pareto frontier of paths between the cluster's boundary vertices can be approximated by a single solution, given a small $\beps$. 
This allows us to introduce a small number of apex-edge pairs that enable our approximate BOSP search algorithm to avoid expanding vertices within the correlated clusters.

\ignore{
\begin{definition}[super-edge set]
    Let $\psi$ be some correlated cluster
    and
    $b_i,b_j \in B(\psi)$ be two boundary vertices of $\psi$.
    Set $\Pi^*_{b_i\rightarrow b_j}$ to be the Pareto-optimal solution set of paths between $b_i$ and $b_j$.
    Let $\hat{E}:=\{\hat{e}\vert \hat{e}=(b_i,b_j)\}$ be a set of new edges connecting $b_i$ and $b_j$.
    We say that $\hat{E}$ is an $\beps$-super-edges set 
    of $b_i$ and $b_j$
    if    
    $\hat{E}$ $\beps$-dominates $\Pi^*_{b_i\rightarrow b_j}$.
\end{definition}
}

\ignore{
\begin{definition}[Cluster Crossing Cost Approximation]
    \label{def:CCCA}
    Let~$\psi$ be some correlated cluster,
    $b_i,b_j \in B(\psi)$ be two boundary vertices of $\psi$
    and~$\beps \geq 0$ be some approximation factor.
    Let~$\Aset$
    be a set of $\beps$-bounded apex-path pairs
    such that 
    $\forall \AP =\langle \bold{A},\pi \rangle \in \Aset$, the path~$\pi$ connects $b_i$ to $b_j$ while only traversing  vertices of $\psi$.
    We say that $\Aset$ is a Cluster Crossing Cost Approximation (CCCA)
    of $b_i$ and $b_j$ in $\psi$ with approximation factor $\beps$
    if    
    the set of representative paths of all apexes of $\Aset$
    $\beps$-dominates the subset of $\Pi^*(b_i,  b_j)$ that only traverses $\psi$.    
\end{definition}

\YH{****************  New Text **************\\}
Fig. \ref{fig:ccca} visualizes a Cluster Crossing Cost Approximation (CCCA) between two boundary vertices $b_i,b_j$ of cluster $\psi$. We start by computing the Pareto-optimal solution set of $\Pi^*(b_i,b_j)$ when considering paths that traverse \emph{only} $\psi$ (Fig \ref{fig:ccca}(a)). Since the path $\pi_2=\{b_i,u,v,b_j\}$ is \beps-dominating the apex $\bold{A}$ and the entire Pareto frontier (Fig \ref{fig:ccca}(b)), \Aset, the CCCA of the path $b_i \rightarrow b_j$, is defined using a single apex-path pair $\AP=~\langle \bold{A}, \pi_2 \rangle$. Consequently (Fig \ref{fig:ccca}(c)), we introduce a new super-edge $\hat{e}=(b_i,b_j)$ whose cost is $\bold{c}(\pi_2)$ and its corresponding apex-edge is $\ApEd=\langle \bold{A},\hat{e} \rangle$.

In practice, any Approximate BOSP solver inevitably needs to compute CCCAs for different partial parts during a search query. However, in the presence of correlated clusters in the graph, CCCAs can be efficiently precomputed between boundary vertices of these clusters, bypassing the need to handle the complex bi-objective search space inside the clusters.

\YH{**************************************\\}
}
Roughly speaking, we need to identify as many clusters as possible while maximizing their size. Large clusters can help reduce the search space by avoiding inner-cluster vertices. However, large clusters have boundary vertices that are far apart, what may lead to a large number of mutually-undominated paths.
\ignore{Roughly speaking, we need to identify as many clusters as possible, ideally having them be as large as possible. Large clusters have boundary vertices that are far apart and may contain a large number of mutually-undominated paths.
}
The preprocessing phase of our framework addresses two key questions:
\begin{enumerate}
    \item[Q1] How can we efficiently detect and delineate correlated clusters within the graph?
    \item[Q2] How can we efficiently compute an approximation of all mutually undominated paths connecting the boundary vertices of a cluster?
\end{enumerate}

\ignore{
\subsection{Correlation-Based Preprocessing}
\label{sec:correlation_based_preprocessing}
In regions with strong correlation,\footnote{\OS{is this high correlation or strong correlation?}} while there may be a large number of solutions in the Pareto-optimal set, they can all be approximated using a small approximation factor. 
Following this insight, we propose a preprocessing framework to identify and leverage the correlation between objectives. Before doing so, we first introduce several key definitions to formalize the discussion.

\begin{definition}[correlated cluster]
    Let $\ell$ be some two-dimensional line 
    (i.e., $\ell: a  x+ b  y +1 = 0$ for some~$a,b$ s.t.~$\sqrt{a^2 + b^2}>0$)
    and
    $\delta >0$ be some threshold.
    Given a graph~$\mathcal{G}=(\mathcal{V},\mathcal{E})$,
    a $(\delta,\ell)$-correlated cluster of \G is a subgraph~$(V,E)$ s.t.
    (i)~$V \subseteq \mathcal{V}$ and $E = (V \times V) \cap \mathcal{E}$
    and
    (ii)~$\forall e \in E$ with costs $\boldsymbol{c}(e)):=(c_1(e), c_2(e))$ it holds that 
    \begin{equation}
        \label{eq:distance_between_edge_and_line}
        \texttt{dist}_\perp(\ell,\boldsymbol{c}(e)) = \frac{\vert ac_1(e)+bc_2(e)+1\vert}{\sqrt{a^2+b^2}} \leq \delta.
    \end{equation}
    Namely, $(V,E)$ contains all edges whose costs are at most~$\delta$ distance from $\ell$.
\end{definition}

\begin{definition}[boundary vertices]
    Let $\psi=~(V_\psi,E_\psi,\ell_\psi,\delta_\psi)$ be a correlated cluster of $\mathcal{G}=(\mathcal{V},\mathcal{E})$. 
    The set of boundary vertices of $\psi$ in \G, denoted as $B(\psi)$, is defined as
    \[
        B(\psi):=\{u\in V_\psi \ \vert \ \exists v\in \mathcal{V} \setminus V_\psi~s.t.~(u,v)\in E \ \text{or} \ (v,u)\in E \}.
    \]    
    In other words, a vertex $u\in V_\psi$ is a boundary vertex if it has at least one adjacent vertex $v\in \mathcal{V} \setminus V_\psi$ via either an outgoing edge $(u,v)$ or incoming edge $(v,u)$.    
\end{definition}

In an $(\ell, \delta)$-correlated cluster of \G,
small values of $\delta$ typically correspond with the fact that the Pareto frontier of paths between the cluster's boundary vertices can be approximated using a small $\beps$. 
Specifically, we are interested in settings where we can approximate all paths between two boundary vertices using a \emph{single} edge.
The following definition formalizes this concept.

\begin{definition}[super-edge set]
    Let $\psi$ be some correlated cluster
    and
    $b_i,b_j \in B(\psi)$ be two boundary vertices of $\psi$.
    Set $\Pi^*_{b_i\rightarrow b_j}$ to be the Pareto-optimal solution set of paths between $b_i$ and $b_j$.
    Let $\hat{E}:=\{\hat{e}\vert \hat{e}=(b_i,b_j)\}$ be a set of new edges connecting $b_i$ and $b_j$.
    We say that $\hat{E}$ is an $\beps$-super-edges set 
    of $b_i$ and $b_j$
    if    
    $\hat{E}$ $\beps$-dominates $\Pi^*_{b_i\rightarrow b_j}$.
\end{definition}

We propose a preprocessing framework to identify correlated clusters within the graph and generate super-edges between their boundary vertices. The core idea is that, for a given user-defined \beps and $\delta$, the expensive \emph{bi}-objective search can be avoided by leveraging precomputed super-edges, which can be efficiently obtained from \emph{single}-objective search queries in the presence of strong correlation between objectives. 

We aim to identify as many clusters as possible, ideally as large as possible. Large clusters have boundary vertices that are far apart, allowing for the creation of super-edges that represent long optimal trajectories in the graph. Consequently, for a bi-objective search query between two vertices, introducing super-edges is expected to speed up the search compared to exploring the original bi-objective search space.

The proposed preprocessing framework addresses two key questions:
\begin{enumerate}
    \item[Q1] How can we efficiently detect and delineate correlated clusters within the graph?
    \item[Q2] How can we efficiently extract super-edges for each identified correlated cluster?
\end{enumerate}
}

\begin{figure}[t]
    \centering
    \includegraphics[scale=0.70]{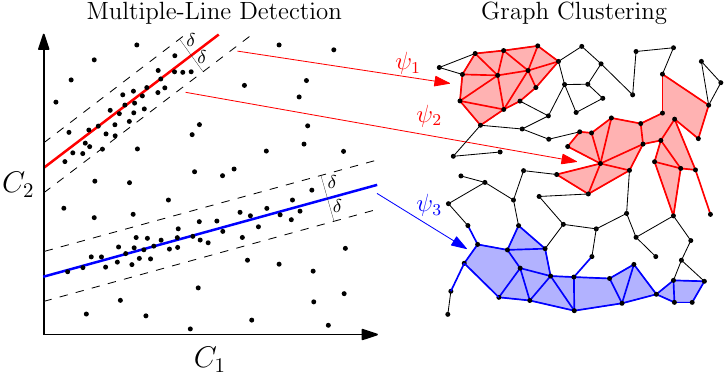}
    \caption{
        Detection and delineation of correlations between the two cost objectives (Q1). 
        \textbf{Left}: a RANSAC-based approach is used to identify distinct linear modes in the 2D objectives space (Alg.~\ref{alg:ransac_lines}). 
        \textbf{Right}: the lines are then used for delineating correlated clusters within the graph. 
    }
    \label{fig:clustering_process}
    \vspace{-3.5mm}
\end{figure}

\subsubsection{Detecting and Clustering  (Q1)}

\begin{algorithm}[tb]
    \caption{Multiple-line detection using RANSAC}
    \label{alg:ransac_lines}
    \textbf{Input}: \begin{enumerate}
        \item[] Graph $\mathcal{G}=(\mathcal{V},\mathcal{E},\boldsymbol{c})$ where $\boldsymbol{c}$ is normalized
%        \item[] Approximation factors vector \beps
        \item[] Allowed distance from the representative line $\delta$
        \item[] Hyperparameters  $n_{\text{hypotheses}}$, $n_\text{min\_inliers}$
    \end{enumerate}
    \textbf{Output}: Set of identified line coefficients $\mathcal{L}$  \\
    \vspace{-1mm}
    \begin{algorithmic}[1] %[1] enables line numbers
        \State $\mathcal{L} \gets \emptyset$ \Comment{Initialize set of detected lines}
        \State $E  \gets \mathcal{E}$ \Comment{Initialize set of all edges}
 %       \State $\text{iter} \gets 0$ \Comment{Initialize iteration counter}
 %       \While{$|\mathcal{S}|>\text{min\_samples}$~\textbf{and}~$\text{iter} <~ \text{max\_iter}$}
 \vspace{1mm}
       \While{\texttt{ToContinue$()$}}
            \State $\mathcal{L}_\text{candidates} \gets \emptyset$ \Comment{Initialize candidate lines set}
            \For{$m=1$ to $n_{\text{hypotheses}}$ }
                \State Sample two random edges $e_i, e_j \in E$
                \State $\ell_m \gets \texttt{LineFit}(e_i,e_j)$
                \State $\mathcal{I}_m \gets \emptyset$ \Comment{Initialize inliers set for $\ell_m$}
                \For{\textbf{each} $e_k \in E$} 
                    \If{$\texttt{dist}_\perp(\ell_m,c(e_k)) \leq \delta$}  \Comment{Eq.~\eqref{eq:distance_between_edge_and_line}}
                        \State $\mathcal{I}_m \gets \mathcal{I}_m \cup e_k$ \Comment{Add  $e_k$ as an inlier}
                    \EndIf
                \EndFor
                %\If{$|\mathcal{I}_m| > \text{min\_inlier\_threshold}$} 
                \If{$|\mathcal{I}_m| > n_\text{min\_inliers}$} 
                    \State $\mathcal{L}_\text{candidates} \gets \mathcal{L}_\text{candidates} \cup \{(\ell_m, |\mathcal{I}_m|)\}$
                \EndIf
            \EndFor
            \If{$\mathcal{L}_\text{candidates} \neq \emptyset$}
                \State Select $\ell^*$ from $\mathcal{L}_\text{candidates}$ with maximal inliers $\mathcal{I}^*$
                \State $\mathcal{L} \gets \mathcal{L} \cup \ell^*$ \Comment{Store best detected line}
                \State $E \gets E \setminus \mathcal{I}^*$ \Comment{Remove inliers from sample set}
            \EndIf
            %\State $\text{iter} \gets \text{iter}+1$
        \EndWhile
        \State \textbf{return} $\mathcal{L}$
    \end{algorithmic}    
\end{algorithm}
%\vspace{-3mm}

Given an input graph, our objective is to detect and delineate correlated clusters whose edges exhibit a strong correlation between objectives. 

To motivate this step, consider a graph~$\mathcal{G}=(\mathcal{V},\mathcal{E})$ containing two perfectly-correlated disjoint subsets $E_1,E_2$ of $\mathcal{E}$. 
Since each set $E_i$ is perfectly correlated, all edge costs of $E_i$ lie on a line $\ell_i$ with parameters~$a_i, b_i$, which may differ. 
For example, time and distance may be perfectly correlated at any constant speed.

Merging $E_1$ and $E_2$ would not only break the perfect correlation in the group, but would also increase the minimal required \beps for approximating the Pareto frontier of paths between boundary vertices using a single solution. 
The same argument holds even when the correlation is not perfect.

To this end, in order to detect distinct linear relationships in the 2-dimensional $(\boldsymbol{C_1},\boldsymbol{C_2})$ space, we utilize RANSAC (Random Sample Consensus) \cite{ransac_1981}, similar to \citet{RANSAC_line_fitting}. RANSAC is an iterative method for estimating model parameters from observed data while distinguishing inliers from outliers. In our case, it is adapted to distinguish between different linear relationships in the objective costs space. 
%This approach can be easily extended to higher objective dimensions and adapted for non-linear correlations. 

Our RANSAC-based multi-line detection algorithm is summarized in Alg.~\ref{alg:ransac_lines}.
It takes as input a graph \G with normalized edge costs\footnote{{Each element in $\boldsymbol{C_1}$ and $\boldsymbol{C_2}$ is divided by $\max(\boldsymbol{C_1})$ and $\max(\boldsymbol{C_2})$, respectively.}}, 
the allowed deviation threshold $\delta$, and two hyperparameters: $n_{\text{hypothesis}}$, the number of tested hypotheses before detecting a line and $n_{\text{min\_inliers}}$, the minimum number of inliers required to accept a detected line.

The algorithm iteratively samples two edges and fits a line through their 2D cost coordinates (Lines 6-7), ensuring a positive slope (i.e., a positive correlation). It then counts inliers - edges that $\delta$-conform with the fitted line (Lines 9-11). The fitted line with the most inliers is selected and added to $\mathcal{L}$, the set of detected lines (Lines 15-16). All corresponding inliers are then removed (Line 17), and the process repeats on the remaining data. This iterative procedure continues until termination conditions are met (Line 3), such as too few edges to sample from or reaching the iterations limit.

{The left pane of Fig.~\ref{fig:clustering_process} illustrates an example of two correlation lines identified using the proposed RANSAC method. Each line has a corresponding subset of cost samples that lie within a distance of up to $\delta$.}

After computing $\mathcal{L}$ which captures the distinct linear relationships between objectives' costs, our next step is to delineate the boundaries of the correlated clusters associated with each line. Inspired by \citeauthor{tarjan_connected_components}~(\citeyear{tarjan_connected_components}), we propose a connected-components labeling algorithm for delineating the correlated clusters based on $\mathcal{L}$.

\ignore{
\begin{algorithm}[tb]
    \caption{Correlation clustering via connected components analysis}
    \label{alg:clusters_cca_detection}
    \textbf{Input}: \begin{itemize}
        \item[] Graph $\mathcal{G}=(\mathcal{V},\mathcal{E},\boldsymbol{c})$ where $\boldsymbol{c}$ is normalized
        \item[] Set of detected lines $\mathcal{L}$
        \item[] Approximation factors~\beps
        \item[] Allowed distance from the representative line $\delta$
        \item[] Hyperparameter $n_{\text{min\_vertices}}$
    \end{itemize} 
    \textbf{Output}: Set of identified correlated clusters $\Psi$  \\
    \begin{algorithmic}[1] %[1] enables line numbers        
        \State $\Psi \gets \emptyset$ 
        \State $\mathcal{UV} \gets \mathcal{V}$ \Comment{Initialize unvisited vertices set}
        \For{\textbf{each} $v \in \mathcal{UV}$} 
            \State $\psi \gets \psi(V_\psi=\emptyset,E_\psi=\emptyset,\ell_\psi=\emptyset,\delta_\psi=\delta)$ 
            \State $\psi \gets \texttt{DFS}(v,\psi,\boldsymbol{\varepsilon},\delta)$ 
            \If{$|V_\psi| \geq n_{\text{min\_vertices}}$}
                    \State $\Psi \gets \Psi \cup \psi$
            \EndIf
        \EndFor
        \State \textbf{return} $\Psi$                       

        \Function{\texttt{CheckLinearRelations}}{$E,\mathcal{L},\boldsymbol{\varepsilon},\delta$}
            \State $L \gets \emptyset$
            \For{\textbf{each} $e \in E$}
                \For{\textbf{each} $\ell \in \mathcal{L}, \ell \notin L$}
                    \If{$\texttt{dist}_\perp(\ell,(c_1(e),c_2(e))) \leq~\delta$}
                        \State $L \gets L \cup \ell$
                    \EndIf
                \EndFor            
            \EndFor
            \State \textbf{return} $L$
        \EndFunction
        
        \Function{\texttt{DFS}}{$v,\psi,\boldsymbol{\varepsilon},\delta$}
        \State $\mathcal{UV} \gets \mathcal{UV} \setminus v$ \Comment{Mark $v$ as visited}
            \State $E \gets $ all incoming and outgoing edges of $v$
            \State $L \gets \textsc{CheckLinearRelations}(E,\mathcal{L},\boldsymbol{\varepsilon},\delta)$
            \If{$|L|=1$} \Comment{All edges adhere to a single line}
                \If{$\ell_\psi = \emptyset$}
                    \State $\ell_\psi \gets L$ \Comment{Set the cluster's line once}
                \EndIf
                \For{\textbf{each} neighbor $u\in\mathcal{UV}$ of $v$}
                    \State $\psi \gets \texttt{DFS}(u,\psi,\boldsymbol{\varepsilon},\delta)$ \Comment{Recursively~extend~$\psi$}
                \EndFor                            
            \EndIf
            \State $\psi \gets \psi \cup v$
            \State \textbf{return} $\psi$
        \EndFunction
    \end{algorithmic}
\end{algorithm}
}

The algorithm maintains a set of unvisited graph vertices and terminates only when all vertices are visited. In each iteration, an unvisited vertex $u$ is randomly selected as the member of a new correlated cluster. 
Then, the algorithm examines all of $u$'s neighboring edges $E_u$. For each edge~$e \in E_u$, we compute the set of lines $L_u \subset \mathcal{L}$ which $e$ conforms with.
If there exists a line $\ell_u$ that all these edges conform to (i.e., $\bigcap_{e \in E_u} L_u \neq \emptyset$) then a new cluster $\psi_u$ is created and $u$ is added to the cluster's vertex set.
Now, a Depth-First Search (DFS) recursion is invoked for each neighboring vertex to expand the cluster. All neighbors are then removed from the set of unvisited vertices. 
This process is recursively repeated until not all of the current vertex's neighboring edges conform to the cluster's line  $\ell_u$. 
Subsequently, a new vertex is randomly chosen from the unvisited vertices set, and the process repeats until all of the graph's vertices are visited.
The right pane of Fig.~\ref{fig:clustering_process} illustrates an example of delineating three correlated clusters based on $\mathcal{L}$.

\ignore{
Alg. \ref{alg:clusters_cca_detection} presents the pseudo-code for the correlated cluster delineation algorithm. The algorithm maintains a set of unvisited nodes (Line 2) and keeps running until this set is empty (Line 3). In each iteration, an unvisited node is selected and treated as the seed for a potentially new correlated cluster (Line 4). The boundaries of this cluster are then recursively expanded using Depth-First Search (DFS) recursive approach (Line 5). Once the recursion completes, the new correlated cluster is retained only if it contains a minimal number of vertices (Lines 6-7). The DFS function uses an auxiliary function to determine the linear correlations exhibited by the currently considered incoming and outgoing edges (Line 19). This function (Lines 9-15) iterates over all considered edges and checks which of the identified lines in $\mathcal{L}$ they align with. The recursion proceeds only if all edges conform to the same linear relationship (Lines 20-22). 
}
\subsubsection{Internal Cluster Cost Approximation (ICCA) (Q2)}
\label{sec:ICCA}

Let~$\psi$ be a correlated cluster with vertices $\mathcal{V}_\psi$ and edges~$\mathcal{E}_\psi$.
For each boundary pair $b_i,b_j \in B(\psi)$ we run \apex with approximation factor $\beps$ on on the graph~$(\mathcal{V}_\psi,\mathcal{E}_\psi)$.
This yields a set of apex-path pairs $\AP^1_{i,j},\ldots ,\AP^n_{i,j}$.

For each such apex-path pair $\AP^k_{i,j}=\langle \bold{A}^k_{i,j},\pi^k_{i,j} \rangle$, we introduce an edge which we call a \emph{super-edge} $\hat{e}^k_{i,j}$ connecting~$b_i$ to $b_j$ and associate it with two cost vectors $\mathbf{c}_\psi, \mathbf{c'}_\psi$ corresponding to the cost of the representative path and the apex-path pair's apex, respectively.

Specifically, we set 
$\mathbf{c}_\psi(\hat{e}^k_{i,j}):= \mathbf{c}(\pi^k_{i,j})$
and
$\mathbf{c'}_\psi(\hat{e}^k_{i,j}):= \mathbf{c}(\AP^k_{i,j})$.
We then set
${\hat{\mathcal{E}}_{\psi,i,j}}$ to be all the super-edges connecting $b_i$ and $b_j$ 
and 
${\hat{\mathcal{E}}_\psi} : = \bigcup_{b_i,b_j \in B(\psi),i \neq j} {\hat{\mathcal{E}}_{\psi,i,j}}$.
{
\begin{property}
\label{prp:super_edges_aprox}
Let $\psi$ be some correlated cluster with its corresponding subgraph 
$\mathcal{G}_\psi = (\mathcal{V}_\psi, \mathcal{E}_\psi)$, and let $b_i,b_j \in B(\psi)$ be two boundary vertices of $\psi$. Let ${\hat{\mathcal{E}}_{\psi,i,j}}$ be the set of super-edges connecting $b_i$ and $b_j$. Then, ${\hat{\mathcal{E}}_{\psi,i,j}}$ is an \beps-approximation of the Pareto-optimal set of any path between~$b_i$ and $b_j$ in $\G_\psi$.    
\end{property}
This property follows directly from the optimality of \apex (Thm. 1 in \cite{zhang2022pex}).
}

\begin{example}
    \label{ex:ccca}

Consider the cluster $\psi$ depicted in Fig.~\ref{fig:ccca},
which contains three  Pareto-optimal solution $\pi_1, \pi_2$ and~$\pi_3$ between $b_i$ and $b_j$ (Fig.~\ref{fig:ccca}(a)). Here, their costs are $(20,100),(80,30)$ and $(90,28)$, respectively. 
When running \apex with an approximation factor of $\beps=[0.1,0.1]$, we obtain that the corresponding Pareto frontier (Fig.~\ref{fig:ccca}(b)) can be approximated using $\pi_1$ and $\pi_2$. In this example, \apex terminated with two apex-path pairs $\AP^1_{ij}, \AP^2_{ij}$.
The first~$\AP^1_{ij}$, is the trivial apex-path pair with both the representative path ($\pi_1$) and the apex having a cost of $(20,100)$ .
The second~$\AP^2_{ij}$, is the results of merging $\pi_2$ and $\pi_3$, with the representative path being $\pi_2$. Here, the apex cost $(80,28)$ which is the element-wise minimum between the costs of $\pi_2$ and $\pi_3$.
super-edges $\hat{e}^1_{ij}$ and $\hat{e}^2_{ij}$ are added between $b_i$ and $b_j$ (Fig.~\ref{fig:ccca}(c)) with costs derived from $\AP^1_{ij}$ and $\AP^2_{ij}$, respectively.
Specifically, $\mathbf{c}_\psi(\hat{e}^1_{ij})=\mathbf{c}'_\psi(\hat{e}^1_{ij})=(20,100)$ and $\mathbf{c}_\psi(\hat{e}^2_{ij})=(80,30)$ while $\mathbf{c}'_\psi(\hat{e}^2_{ij})=(80,28)$.

\end{example}

\begin{figure}[t]
    \centering
    \includegraphics[scale=0.62]{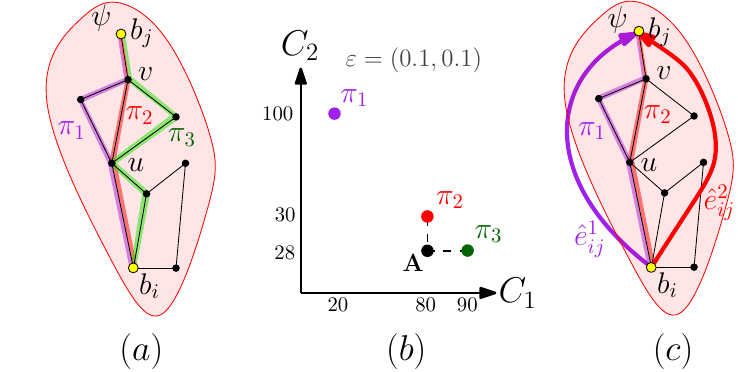} % 0.577
    \caption{
        Introducing super-edges
        connecting the boundary vertices $b_i$ and $b_j$ in cluster $\psi$. 
        See Example~\ref{ex:ccca} for details.
        }
    \label{fig:ccca}
    \vspace{-4mm}
\end{figure}

\ignore{
\YH{alternative text:\\}
After obtaining $\Psi$, a set of all correlated clusters in~\G, our next step is to perform \emph{Internal Cluster Cost Approximation (ICCA)}, namely, computing the CCCA for every boundary vertices pair of a cluster.
%Namely, we aim to approximate the Pareto-optimal solutions sets of all paths between every pair of boundary vertices. 
%
Specifically, we will compute the apex-edge pairs set SE$(\psi)$ which every apex-edge in it is based on all boundary vertices pairwise CCCA's of cluster~$\psi$:
\[
    \{\Aset \ \vert \ b_i,b_j \in B(\psi), i\neq j\}
\]
In other words, SE$(\psi)$ holds the apex-edge pairs representation of all super-edges of cluster $\psi$.
As we will see, adding SE$(\psi)$ to \G,  will allow to remove all non-boundary vertices of clusters.

\YH{origianl text:\\}
After obtaining $\Psi$, a set of all correlated clusters in~\G, our next step is to perform \emph{Internal Cluster Cost Approximation (ICCA)}. Namely, we aim to approximate the Pareto-optimal solution sets of all paths between every pair of boundary vertices. 
Specifically, we will compute the apex-edge set SE$(\psi)$ of all super-edges of $\psi$. This set is based on aggregating all of \Aset of every boundary vertices pair~$(b_i,b_j)$. As we will see, adding edges of SE$(\psi)$ as apex-edge pairs to \G,  will allow to remove all non-boundary vertices of clusters.
%Our motivation is to retain a cluster's boundary vertices and its corresponding super-edges while removing all of its internal vertices and edges. Consequently, the total number of vertices of \G will be reduced, hopefully resulting in speed-ups in the query phase.

To generate SE$(\psi)$ of cluster $\psi$, we propose the following two-step strategy following the lightweight and the more complex approaches defined in  Sec. \ref{sec:approximating_PF_with_single_objective_search} and \ref{sec:apex}, respectively, for approximating $\Pi^*$.

\OS{I don't like the term ``generate'' here. We need to define SE properly and say what we add to it (edges / super edges / edge-apex pairs)}

\begin{enumerate}
    \item[S1] 
    We begin by running two All-Pairs-Shortest-Path Dijkstra single-objectives searches (one for each objective) between all boundary vertices pairs of $\psi$ while only considering vertices of $\psi$. For every pair $b_i,b_j \in B(\psi)$, we follow the approach outlined in Sec.~\ref{sec:approximating_PF_with_single_objective_search}. Namely, we compute the extreme solutions \pitl and \pibr and derive $\varepsilon^{\texttt{tl}}$ and $\varepsilon^{\texttt{br}}$. Given the user-provided approximation factor $\beps=\left(\varepsilon_1,\varepsilon_2\right)$, we check whether $\varepsilon^{\texttt{tl}} \leq \varepsilon_2$ or $\varepsilon^{\texttt{br}} \leq \varepsilon_1$ is satisfied.

    If so, we generate a new super-edge based on the extreme solution whose inequality was satisfied. We create an apex-edge pair representation $\ApEd=\langle \bold{EA},\hat{e} \rangle$ where $\hat{e}$ is the extreme solution's path and $\bold{EA}$ is the apex of the two extreme solutions (see Fig. \ref{fig:extreme}(a)).
    
    \ignore{
    \YH{original:\\}
    We begin by running two All-Pairs-Shortest-Path Dijkstra single-objectives searches (one for each objective) between all boundary vertices pairs of $\psi$. For every pair $b_i,b_j \in B(\psi)$, we follow the approach outlined in Sec.~\ref{sec:approximating_PF_with_single_objective_search}. Namely, we check whether one extreme solution can \beps-approximate the other. 

    If so, 
    %let $\pi_{\text{extreme}}$ be the corresponding extreme solution and $\beps_{\text{apx}}$ be the correspond 
    we generate a new apex-edge pair $\ApEd=\langle \bold{EA},\hat{e} \rangle$ where $\hat{e}$ is the extreme solution path and $\bold{EA}$ is the apex of the two extreme solutions (see Fig. \ref{fig:extreme}(a)).}
  
    \item[S2] 
    If the previous approach fails, namely, $\vert\Aset\vert>~1$ (see Fig. \ref{fig:extreme}(b)), we follow the approach outlined in Sec.~\ref{sec:apex}. Specifically, we compute $\Pi^*_{\beps}(b_i,b_j)$ for every pair $b_i,b_j \in B(\psi)$ using \apex while considering only vertices of $\psi$. We generate an apex-edge pair based on each apex-path pair in the query's solution set.
    
\end{enumerate}
At the end of this process, we obtain the set of apex-edge pairs generated for every cluster $\psi\in\Psi$.

\textbf{Note.}
One may consider an alternative methods in case step S1 fails. E.g., we can replace step S2 by iteratively reducing $\delta$ and re-delineating the correlated cluster until the required \beps is sufficiently small for a single super-edge to \beps-approximate $\Pi^*(b_i,b_j)$.
We leave the study of such alternative methods for future work.

\ignore{
\textbf{Refining the cluster boundary}: Empirical observations suggest that $\delta$ and~\beps are inherently related: decreasing $\delta$ reduces the required $\varepsilon$ for generating a single super-edge. Consequently, we can iteratively reduce $\delta$ and re-delineate the correlated cluster until the required $\varepsilon$ is sufficiently small for a single super-edge to $\varepsilon$-approximate $\Pi^*_{b_i\rightarrow b_j}$.
}

\ignore{
\YH{Considering the new content above, we can safely remove the following paragraph:\\}
In our work, we assumed that all generated super-edges correctly \beps-approximate the Pareto-optimal set of paths between boundary vertices pairs. This assumption is empirically verified in our experiments. 
}

}
\subsection{Query Phase}
\label{sec:query_phase}

Recall that in the query phase, 
we assume to have the graph~$\mathcal{G}=(\mathcal{V},\mathcal{E})$
and the user-provided  approximation factor \beps
as well the set of correlated clusters~$\Psi$ generated during the preprocessing phase.
Given a query~$\vs, \vt \in \mathcal{V}$
we wish to efficiently compute~$\Pi^*_{\beps}(\vs,\vt)$.

We start by defining a new graph \Gtilde which we call the \emph{query graph}. \Gtilde, which will be implicitly constructed, contains super-edges that avoid having a search algorithm enter correlated clusters that do not include \vs and \vt.

Each edge in the query graph will be associated with two cost functions which will induce a generalized query graph such that running \eapex on this generalized query graph will allow us to efficiently compute $\Pi^*_{\beps}$.

Unfortunately, the branching factor of vertices in \Gtilde may turn out to be quite large. Thus, we continue to describe how to use standard algorithmic practices to deal with it.

\subsubsection{Query Graph}
It will be convenient to assume that every vertex $v \in \mathcal{V}$ belongs to a correlated cluster.
If $v$ was not assigned a cluster in the preprocessing phase, we will assign it with a \emph{trivial cluster} $(\{ v\}, \emptyset, \ell_v, \delta_v)$ containing only $v$ and no edges\footnote{In a trivial cluster
$(\{ v\}, \emptyset, \ell_v, \delta_v)$, the parameters $\ell_v \text{ and } \delta_v$ are meaningless  and any value can be used. {Moreover, a trivial cluster has no super-edges as, it only contains one vertex, which we consider as a boundary vertex.}} and add the cluster to $\Psi$.

Let $\psi_s$ and $\psi_t$ denote the correlated clusters containing~\vs and \vt, respectively. 
We define the query graph~$\tilde{\mathcal{G}} = (\tilde{\mathcal{V}}, \tilde{\mathcal{E}})$ as follows:
\begin{equation*}
\label{eq:V_tilde}
        \tilde{\mathcal{V}} = 
            \underbrace{\left( V_{\psi_s}, \cup \ V_{\psi_t} \right)}_{(\diamondsuit)}
            \cup
            \underbrace{
            \left(
            \cup_{\psi\in \Psi \setminus \{\psi_s \psi_t \}} B(\psi)
            \right)}_{(	\heartsuit)},
\end{equation*}
\vspace{-2.5mm}
\begin{equation*}
\label{eq:E_tilde}
    \tilde{\mathcal{E}} = 
        \underbrace{\left( 
            \mathcal{E} \setminus 
            \{E_\psi~\vert~\psi \in \Psi \setminus \{\psi_s, \psi_t \} \}
        \right)}_{(\clubsuit)}
        \cup
        %\left( E_{\xi_S}, \cup E_{\xi_g} \right)
        % \cup    
        \underbrace{\left(
          \cup_{\psi\in \Psi \setminus \{\psi_s, \psi_t \}} {\hat{\mathcal{E}}_\psi}
        \right)}_{(\spadesuit)}.
        %\cup
        % \left(\mathcal{E} \setminus \{E_\psi\ |\ \psi \in \Psi\} \right).
\end{equation*}

Namely, 
the vertices $\tilde{\mathcal{V}}$ 
include 
$(\diamondsuit)$~all vertices of clusters~$\psi_s$  and $\psi_t$
and
$(	\heartsuit)$~all boundary vertices of the other clusters. 
The edges $\tilde{\mathcal{E}}$ 
include 
$(\clubsuit)$~all edges between clusters as well as all edges of clusters $\psi_s$  and $\psi_t$
and
$(\spadesuit)$~all the super-edges of the clusters that are not $\psi_s$  and $\psi_t$.

We are now ready to define the \emph{generalized query graph} 
$(\tilde{\mathcal{V}}, 
    \tilde{\mathcal{E}},
    \tilde{\mathbf{c}}, \tilde{\mathbf{c}}')$
corresponding to query graph.
The only thing we need to describe are the edge costs functions $\tilde{\mathbf{c}}$ and $\tilde{\mathbf{c}}'$.
For each \emph{original edge} $e \in \mathcal{E}$ we set
$\tilde{\mathbf{c}}(e) := \mathbf{c}(e)$
and
$\tilde{\mathbf{c}}'(e) := \mathbf{c}(e)$. 
For each \emph{super-edge} $\hat{e} \in \hat{\mathcal{E}}_\psi$ of cluster $\psi$ we set
$\tilde{\mathbf{c}}(\hat{e}) := \mathbf{c}_\psi(\hat{e})$
and
$\tilde{\mathbf{c}}'(\hat{e}) := \mathbf{c}'_\psi(\hat{e})$. 

In the following example, we detail the part of the query graph corresponding to the correlated cluster depicted in Fig.~\ref{fig:ccca} and the paths described in Example~\ref{ex:ccca}.

\begin{example}
    Consider the cluster $\psi$ detailed in Example~\ref{ex:ccca} and assume that neither \vs nor \vt are in $V_\psi$.
    Let us consider the contribution of $\psi$ to the query graph $\Gtilde=(\mathcal{\tilde{V}},\mathcal{\tilde{E}})$.
    First, all boundary vertices of $\psi$ such as $b_i$ and $b_j$ will be added to~$\mathcal{\tilde{V}}$ while internal vertices such as $u,v \in \mathcal{V}_\psi$ will not.
    Second, the super-edges of $\psi$ such as $\{\hat{e}^1_{ij},\hat{e}^2_{ij}\}$ are added to~$\mathcal{\tilde{E}}$ as well as edges connecting boundary vertices of $\psi$ to vertices not in $\psi$.
    On the other hand, internal cluster's edges, as $(u,v)$, are removed. 
    Now, recall that before we can run \eapex, we construct the generalized query graph $(\mathcal{\tilde{V}},\mathcal{\tilde{E}},\tilde{\mathbf{c}},\tilde{\mathbf{c}}')$. 
    For super-edges like $\hat{e}^2_{ij} \in \hat{\mathcal{E}}_\psi$, the costs will be $\tilde{\mathbf{c}}(\hat{e}^2_{ij})=\mathbf{c}_\psi(\hat{e}^2_{ij})=(80,30)$ and $\tilde{\mathbf{c}}'(\hat{e}^2_{ij})=\mathbf{c}'_\psi(\hat{e}^2_{ij})=(80,28)$.
    \ignore{
    Any regular edge $e \in \tilde{\mathcal{E}}$ will have a common cost of $\tilde{\mathbf{c}}(e)=\tilde{\mathbf{c}}'(e)=\mathbf{c}(e)$, whereas for super-edges like $\hat{e}^2_{ij} \in \hat{\mathcal{E}}_\psi$, the costs will be $\tilde{\mathbf{c}}(\hat{e}^2_{ij})=\mathbf{c}_\psi(\hat{e}^2_{ij})=(80,30)$ and $\tilde{\mathbf{c}}'(\hat{e}^2_{ij})=\mathbf{c}'_\psi(\hat{e}^2_{ij})=(80,28)$.}
\end{example}
\ignore{
\begin{lemma}
\label{eq:lemma:cluster_eps_approx} 
    For any correlated cluster $\psi$ and any two boundary vertices $b_i,b_j \in B(\psi)$, the Pareto-optimal solution set $\Pi^*(b_i,b_j)$ in $(V_\psi, E_\psi, \mathbf{c}_\psi)$ is an \beps-approximation set of the Pareto-optimal solution set $\Pi^*(b_i,b_j)$ in $(V_\psi, E_\psi, \mathbf{c'}_\psi)$.
\end{lemma}
\begin{proof}
    By the construction of $\bold{c}_\psi$ and $\bold{c'}_\psi$ (for both regular edges and super-edges) it holds that $\forall e \in E_\psi$:
    $\bold{c'}_\psi(e) \preceq \bold{c}_\psi(e)$ as well as $\bold{c}_\psi(e) \preceq_{\beps} \bold{c'}_\psi(e)$. From the latter and from (\ref{eq:edge_cost_constr}) in Lemma \ref{lemma:apx-bound} we derive that $\beps \geq \underset{e\in \mathcal{E_\psi}}{\max} \left(\mathbf{c}_\psi(e) / \mathbf{c'}_\psi(e) - 1\right)$. This satisfies the conditions of Lemma \ref{lemma:apx-bound}. Thus, $\Pi^*(b_i,b_j)$ in $(V_\psi, E_\psi, \mathbf{c}_\psi)$ is an \beps-approximation of $\Pi^*(b_i,b_j)$ in $(V_\psi, E_\psi, \mathbf{c'}_\psi)$.
\end{proof}
}

The following theorem summarizes the correctness of the generalized query graph construction.

\begin{theorem}
\label{thm:query}
Let \vs and \vt be the start and target vertices, respectively, of a search query.
Running \eapex on the generalized query graph \Gtilde yields an \beps-approximation of~$\Pi^*$ in~\G.
\end{theorem}
{The proof builds on Thm. \ref{thm:gapex} and Property \ref{prp:super_edges_aprox}. The full proof is provided in Appendix \ref{apndx_thm_query}.}

\subsubsection{Lazy Edge Expansion}
\label{sec:lazy_edge_expansion}
Recall that within a correlated cluster $\psi$, we connect all pairs of boundary vertices $b_i, b_j \in~ B(\psi)$ by one or more super-edges of the set $\hat{\mathcal{E}}_{\psi,i,j}$. Namely, the number of super-edges introduced is at least quadratic in the number of boundary vertices. Thus, the branching factor of \Gtilde may dramatically increase.
Unfortunately, large branching factors are known to dramatically slow down search-based algorithms (even single-objective ones)~\cite{korf1985depth, edelkamp1998branching}.

To this end, we endow our search algorithm with a lazy edge-expansion strategy~\cite{yoshizumi2000partial} for the super-edges.
Specifically, we maintain two edge lists for each vertex: regular edges and super-edges, both ordered lexicographically from low to high using the edge's $\bold{f}$-value.
When expanding a node, all successors derived from regular edges are pushed to \open as before. For super-edges, however, we follow a \emph{partial expansion} approach: we iterate over super-edges in increasing lexicographic $\bold{f}$-value order and stop as soon as the first successor (originating from a super-edge) is inserted into \open.
When a boundary vertex is popped from \open, we expand the \emph{next}-best super-edge of its predecessor. I.e., the next unprocessed super-edge from the super-edges list of the predecessor.
We refer to the adaptation of \eapex as described above as Partial Expansion \eapex (\pegapex).

\ignore{
Let $\psi_s$ and $\psi_t$ denote the correlated clusters containing~\vs and \vt, respectively. 
We define the query graph~$\tilde{\mathcal{G}} = (\tilde{\mathcal{V}}, \tilde{\mathcal{E}})$ as follows:
\begin{equation*}
\label{eq:V_tilde}
        \tilde{\mathcal{V}} = 
            \underbrace{\left( V_{\psi_s}, \cup \ V_{\psi_t} \right)}_{(\diamondsuit)}
            \cup
            \underbrace{
            \left(
            \cup_{\psi\in \Psi \setminus \{\psi_s \psi_t \}} B(\psi)
            \right)}_{(	\heartsuit)},
\end{equation*}
\begin{equation*}
\label{eq:E_tilde}
    \tilde{\mathcal{E}} = 
        \underbrace{\left( 
            \mathcal{E} \setminus 
            \{E_\psi~\vert~\psi \in \Psi \setminus \{\psi_s, \psi_t \} \}
        \right)}_{(\clubsuit)}
        \cup
        %\left( E_{\xi_S}, \cup E_{\xi_g} \right)
        % \cup    
        \underbrace{\left(
          \cup_{\psi\in \Psi \setminus \{\psi_s, \psi_t \}} \text{SE}(\psi)
        \right)}_{(\spadesuit)}.
        %\cup
        % \left(\mathcal{E} \setminus \{E_\psi\ |\ \psi \in \Psi\} \right).
\end{equation*}

Namely, 
the vertices $\tilde{\mathcal{V}}$ 
include 
$(\diamondsuit)$~all vertices of clusters~$\psi_s$  and $\psi_t$
and
$(	\heartsuit)$~all boundary vertices of the other clusters. 
The apex-edge pairs set $\tilde{\mathcal{E}}$ 
include 
$(\clubsuit)$~all {apex-edge pairs set} between clusters as well as all {apex-edge pairs set} of clusters $\psi_s$  and $\psi_t$
and
$(\spadesuit)$~all the {apex-edge pairs set} of the clusters that are not $\psi_s$  and $\psi_t$.
\YH{More accurate:all the {apex-edge pairs sets SE$(\psi)$} of clusters that are not $\psi_s$  and $\psi_t$.}

\pagebreak
We now prove that running \apex over the query graph~\Gtilde generates a correct approximation of the Pareto-optimal solution set over the original graph \G.

\begin{theorem}
\label{thm:query}
Let \vs and \vt be the start and target vertices, respectively, of a search query.
Running \eapex on the generalized graph~\Gtilde, yields an \beps-approximation of $\Pi^*$ in \G.
\end{theorem}

Similar to Sec.~\ref{sec:apex}, the proof here is not complicated but we ommit details due to lack of space. However, it relies on Thm.~\ref{thm:gapex} and the following Lemma:

\OS{arguments may need to be massaged}
\begin{lemma}
    For any correlated cluster $\psi$ and any two boundary vertices $b_i,b_j \in B(\psi)$, the Pareto-optimal solution set of $(V_\psi, E_\psi, \mathbf{c})$ is an \beps-approximation set of the Pareto-optimal solution set of $(V_\psi, E_\psi, \mathbf{c'}_\psi)$.
\end{lemma}

\ignore{
\begin{proof}
By definition, \Gtilde is derived from \G by replacing a subset of correlated clusters with their boundary vertices and super-edges apex-edges pairs. When \apex extends an apex-path pair \AP, we need to distinguish between two cases:
\begin{enumerate}
    \item Extending with a regular edge $e$: By construction, $e \in \mathcal{V}$, since all regular edges of \Gtilde are the same genuine edges of $G$. Thus, in this case, \apex on \G and \Gtilde will behave the same.
    \item Extending with a super-edge $\hat{e}=(b_i,b_j)$: By construction, \apex will consider all the apex-edges pairs derived from \Aset. By definition, \Aset is an \beps-approximation of the Pareto-optimal solution set $\Pi^*(b_i,b_j)$ in \G of paths traversing only $\psi$. Partial paths that go outside of $\psi$ are either comprised of regular edges, which are the same for \G and \Gtilde, or comprised of other clusters' super-edges, which are \beps-dominated by definition.    
\end{enumerate}
For either type of edge, \apex extends \AP while preserving its \beps-bounding property (Lemma \ref{lemma:eps-bound}). Consequently, every apex-path pair in \open is \beps-bounded as \apex operates on~\G. Thus, running \apex on \Gtilde yields an \beps-approximation of $\Pi^*(\vs,\vt)$ in \G.
\end{proof}
}

}

\ignore{
\OS{Original text\\}
During the evaluation of the preprocessing phase, we observed that in large correlated clusters there may be significantly more super-edges than regular edges for the same cluster, i.e.,  $|E_\psi| \ll |\text{SE}(\psi)|$. This substantial increase in the branching factor poses a challenge for any best-first search algorithm, as the increased branching factor will impair the running time of the search algorithm due to the costly heap operations involved in maintaining the \open list.

To "address" ``this'' issue, we adopt a lazy edge expansion strategy \cite{yoshizumi2000partial} for mitigating the impact of the increased branching factor. Although we use \apex for the empirical evaluation of our method, the following strategy is general in the sense that it is applicable to any best-first search algorithm: 

\begin{itemize}
    \item Two lists of edges are maintained for each vertex: regular edges and super-edges. The super-edges list is sorted in ascending lexicographic order based on the $\bold{f}$-value (in ascending order) of the successor nodes reached via the super-edge.
    \item When expanding a node, all successors derived from regular edges are pushed to \open as before. For super-edges, however, we follow a \emph{partial expansion} approach: we iterate over super-edges in increasing $\bold{f}$-value lexicographic order and stop as soon as the first successor (originating from a super-edge) is inserted into \open. 
    \item When a boundary vertex is popped from \open, we backtrack to its predecessor and expand the \emph{next}-best super-edge (i.e., the next unprocessed super-edge in ascending lexicographic $\bold{f}$-value order). This process continues until another successor from a super-edge is added to \open.
\end{itemize}

This approach enables a controlled and selective expansion of super-edges, preventing state explosion caused by an increased branching factor. We refer to this algorithm, which integrates lazy partial expansion into \apex, as \peapex (Partial Expansion \apex).
}

\section{Evaluation}
\label{sec:eval}
We implemented our algorithms using a combination of Python and C++\footnote{\url{https://github.com/CRL-Technion/BOSP-PE-GApex}.}. We ran all experiments on an HP ProBook 440 G8 Notebook with 16GB of memory. The \apex and \pegapex algorithms were implemented based on \apex original C++ implementation\footnote{\url{https://github.com/HanZhang39/A-pex}.}. 
All experiments were executed on the \texttt{NY}, \texttt{COL}, \texttt{NW} and \texttt{CAL} DIMACS instances, which contain between 250K and 1.9M vertices. 

As an optimization step for the ICCA process (Q2, Sec.~\ref{sec:correlation_based_preprocessing}), we employed a simple and efficient method for approximating $\Pi^*(b_i,b_j)$ without calling \apex (as described in Sec. \ref{sec:ICCA}). Specifically, we ran two single-objective Dijkstra shortest-path queries, one for each objective, for the query $b_i \rightarrow b_j$ considering only the subgraph of cluster $\psi$. We then checked if the user-provided approximation factors is sufficient for \beps-dominating $\Pi^*(b_i,b_j)$ with a single solution. If so, this solution was used to obtain $\hat{\mathcal{E}}_{\psi,i,j}$. This straightforward step is usually one to two orders of magnitude faster than running \apex directly.

\subsection{Correlation-Based Preprocessing on DIMACS}
\label{subsec:dimacs}
\begin{table} 
    \centering   
    \scriptsize
    \begin{tabular}{|c|c|c|c|c|c|c|c|c|}
        \hline
        \rule{0pt}{2.0ex} Instance & $|\mathcal{V}|$ & $|\tilde{\mathcal{V}}|$ & $|\mathcal{E}|$ & $|\tilde{\mathcal{E}}|$ & $b(\G)$ & $b(\Gtilde)$ & \shortstack[b]{Time \\ {[sec]}} & \shortstack[b]{Space \\ {[GB]}} \rule[-1.8ex]{0pt}{0pt} \\
        \hline
        {\footnotesize\texttt{NY}} & $26$ & $4.3$ & $73$ & $110$ & $2.8$ & $25.7$ & $39$ & $0.5$ \\
        \hline
        {\footnotesize\texttt{COL}} & $43$ & $8.2$ & $106$ & $93$ & $2.4$ & $11.2$ & $47$ & $0.5$ \\
        \hline
        {\footnotesize\texttt{NW}} & $121$ & $18$ & $284$ & $227$ & $2.4$ & $12.7$ & $127$ & $0.9$ \\
        \hline
        {\footnotesize\texttt{CAL}} & $189$ & $25$ & $466$ & $360$ & $2.5$ & $14.3$ & $208$ & $1.7$ \\
        \hline        
    \end{tabular}    
    \caption{Comparison of the size of \G (original graph) and \Gtilde (query graph) for~$\boldsymbol{\varepsilon} = [0.01, 0.01]$, including the number of vertices and edges (in tens of thousands), average branching factor $b$, and preprocessing time and space usage.}    
    \label{tbl:g_vs_gtilde}
    \vspace{-4.5mm}
\end{table}
Recall that the DIMACS dataset is a standard benchmark in the field and is supposed to simulate real-world data. Thus, we start by reporting how our framework behaves on this dataset.
Tbl.~\ref{tbl:g_vs_gtilde} compares the original graph \G and the query graph \Gtilde in terms of size (vertices, edges, average branching factor), as well as preprocessing time and space required for storing the optimal paths abstracted by super-edges.
As expected, \Gtilde consistently has dramatically fewer vertices but a higher branching factor when compared to \G. However, these two factors counterbalance each other, and both graphs have comparable number of edges.

We continue to visualize the objective correlation and how it manifests in our framework for the \texttt{NY} and \texttt{CAL} DIMACS instances.
Plotting the edge costs as points in the bi-objective space (Fig.~\ref{fig:combined_figures_NY_CAL_preprocessing_demo}a,\ref{fig:combined_figures_NY_CAL_preprocessing_demo}c), we can see that the entire bi-objective space can be decomposed into four disjoint, highly-correlated linear relationships---a pattern observed consistently across the DIMACS dataset.
These four modes of correlation were detected by the RANSAC method (Alg.~\ref{alg:ransac_lines}).
Importantly, each mode needs to be further subdivided into correlated clusters which are depicted in Fig.~\ref{fig:combined_figures_NY_CAL_preprocessing_demo}b,\ref{fig:combined_figures_NY_CAL_preprocessing_demo}d.

\subsection{Lazy Edge Expansion Ablation Study}
As demonstrated in Sec.~\ref{subsec:dimacs}, \Gtilde's branching factor is much larger than \G's which is why we suggested a lazy edge-expansion strategy (Sec.~\ref{sec:query_phase}). 
To this end, we compare (Fig.~\ref{fig:ablation_study}) 
the query execution times of \eapex and \pegapex on \Gtilde, across various DIMACS instances for an approximation factor of $\beps=[0.01,0.01]$. \pegapex outperforms \eapex for almost all instances with the speed up in query times reaching above $5\times$. 

\begin{figure*}
    \centering
    \begin{subfigure}[t]{0.49\textwidth}  % Reduce width slightly if necessary
        \centering
        \includegraphics[width=1\textwidth]{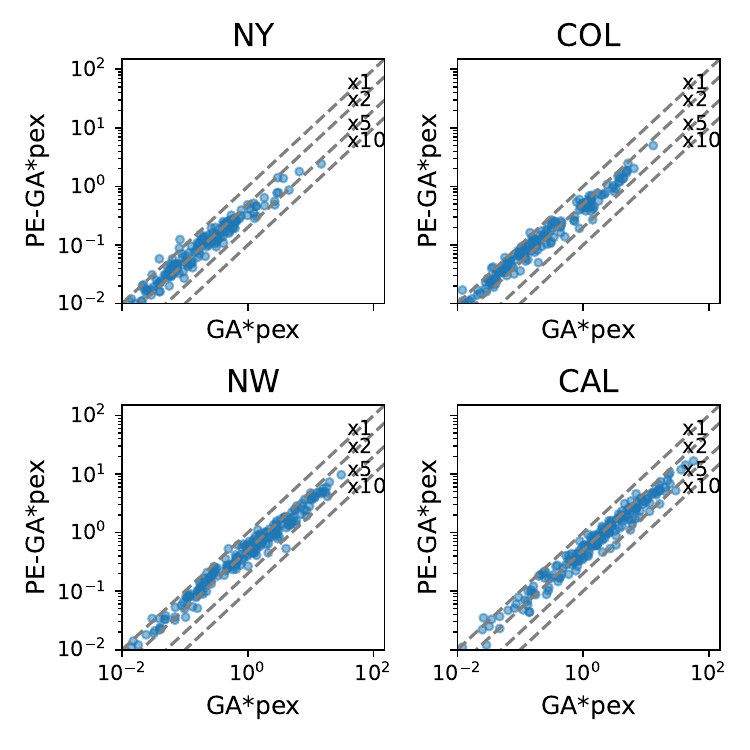}
        \vspace{-7.5mm}
        \caption{}         
        \label{fig:ablation_study}
    \end{subfigure}
    \hfill
    \begin{subfigure}[t]{0.49\textwidth}
        \centering
        \includegraphics[width=1\textwidth]{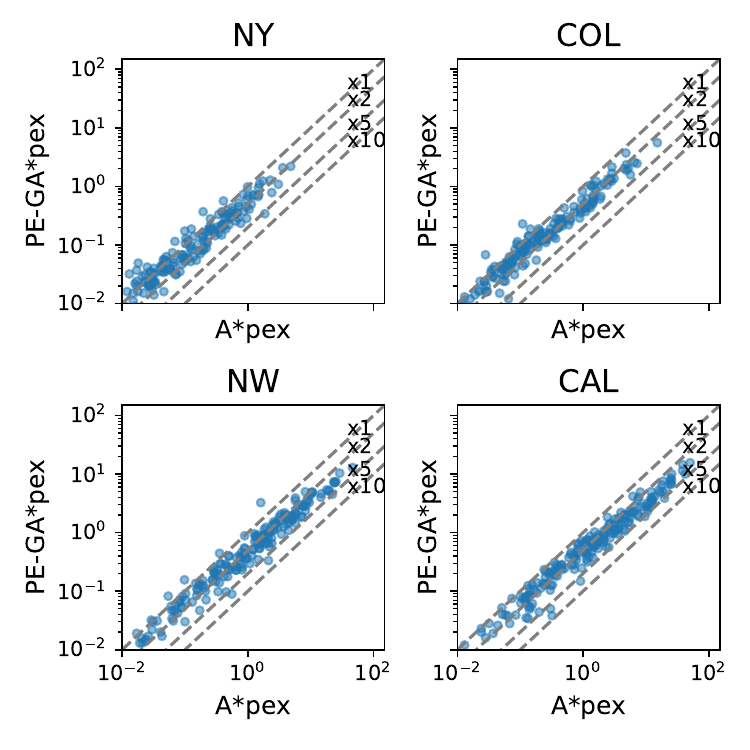}
        \vspace{-7.5mm}
        \caption{}    
        \label{fig:running_times_comparison}
    \end{subfigure}
    \vspace{-2.5mm}
    \caption{Running times (in seconds) on different queries and DIMACS instances for $\boldsymbol{\varepsilon} = [0.01, 0.01]$.
    (a)~Ablation study---comparing \pegapex with \eapex.
    (b)~Approach evaluation---comparing \pegapex with \apex.
    }
    \vspace{-5mm}
    \label{fig:combined_figures}
\end{figure*}

\begin{figure}[H]
    \centering
    \begin{subfigure}{0.23\textwidth}  % Reduce width slightly if necessary
        \centering
        \includegraphics[width=1\textwidth]{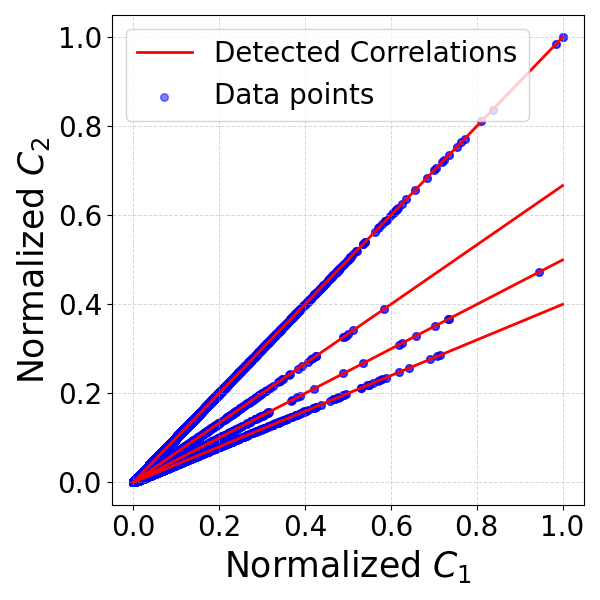}
        \caption{}         
        \label{fig:CAL_RANSAC_demo}
    \end{subfigure}
    \hfill
    \begin{subfigure}{0.23\textwidth}
        \centering
        \includegraphics[width=1\textwidth]{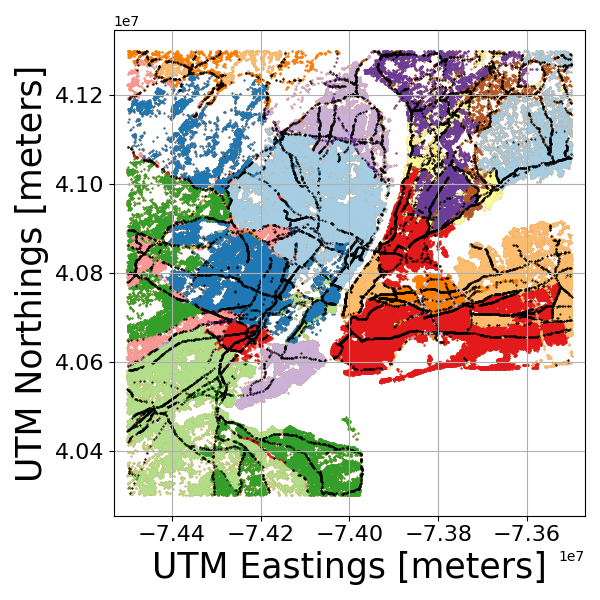}
        \caption{}         
        \label{fig:CAL_map_clusters_demo}
    \end{subfigure}
    \begin{subfigure}{0.23\textwidth}  % Reduce width slightly if necessary
        \centering
        \includegraphics[width=1\textwidth]{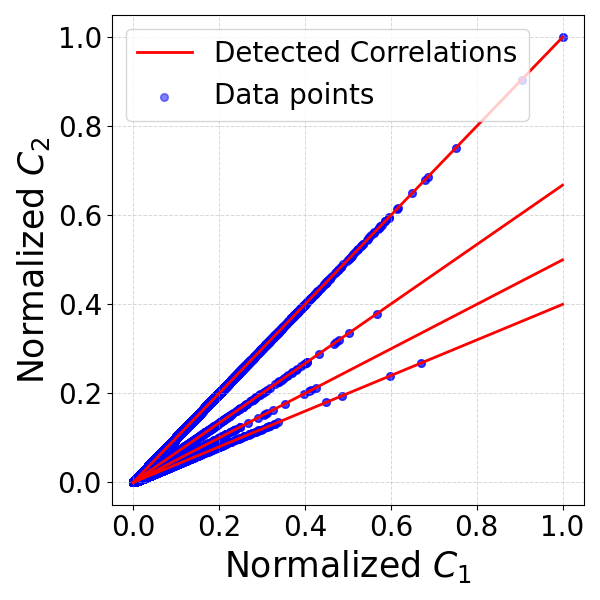}
        \caption{}         
        \label{fig:CAL_RANSAC_demo}
    \end{subfigure}
    \hfill
    \begin{subfigure}{0.23\textwidth}
        \centering
        \includegraphics[width=1\textwidth]{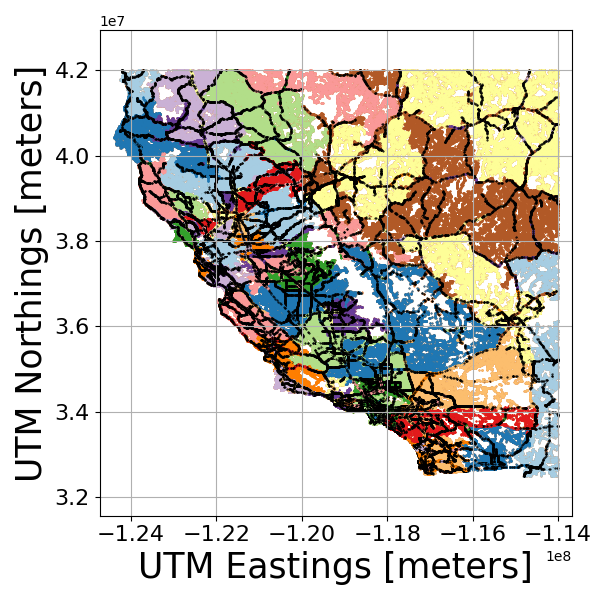}
        \caption{}         
        \label{fig:CAL_map_clusters_demo}
    \end{subfigure}
    
    \caption{\textbf{(a)+(c)}
            Edge cost plotted on the 2D objective costs (blue dots) and linear correlations computed by Alg.1 (red lines) for the \texttt{NY} and \texttt{CAL} instances, respectively.
            \textbf{(b)+(d)} Geo-spatial display of \texttt{NY}'s and \texttt{CAL}'s correlated clusters (plotted as color patches), respectively, computed using the correlation clustering method (Sec. 5.1). Each cluster's boundary vertices are marked in black dots.}           
    \label{fig:combined_figures_NY_CAL_preprocessing_demo}    
\end{figure}

\ignore{
\begin{figure}[t]
            \centering
            \includegraphics[scale=0.245] {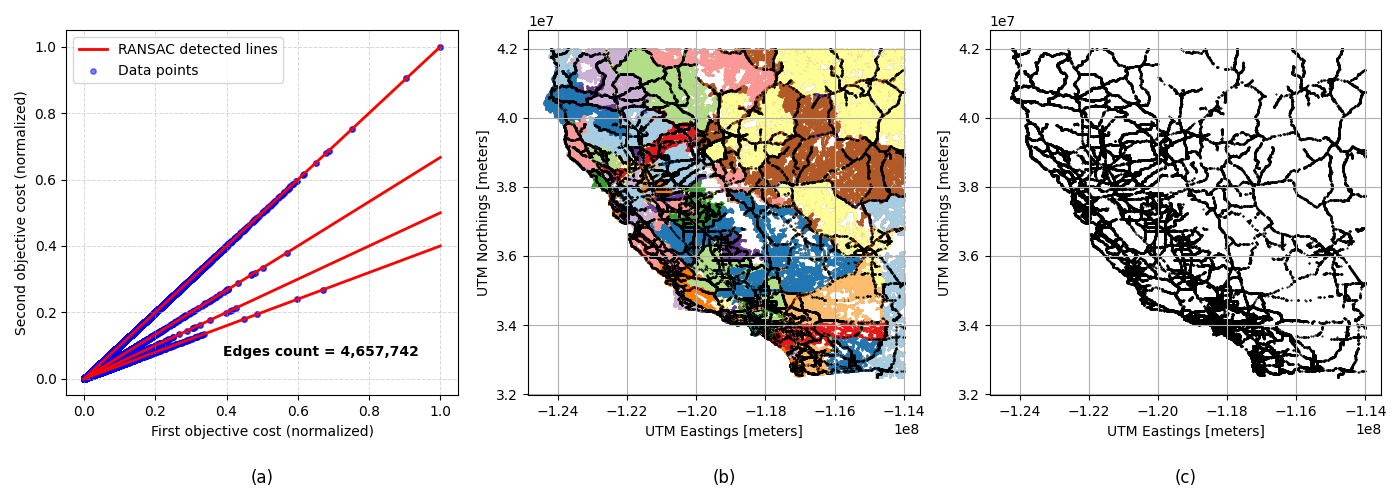}
            \caption{
            \textbf{(a)} 
            Edge cost plotted on the 2D objective costs (blue dots) and linear correlations computed by Alg.\ref{alg:ransac_lines} (red lines) for the \texttt{CAL} instance.
            \textbf{(b)} Geo-spatial distribution of correlated clusters (plotted as different color patches) computed using the correlation clustering method (Sec.~\ref{sec:correlation_based_preprocessing}). Each cluster's boundary vertices are marked in black dots.}          
            \label{fig:clusters_on_map}
\end{figure}}

% We continue to visualize the objective correlation and how it manifests in our framework for the \texttt{NY} and \texttt{CAL} DIMACS instances.
% %
% Plotting the edge costs as points in the bi-objective space (Fig.~\ref{fig:combined_figures_NY_CAL_preprocessing_demo}a,\ref{fig:combined_figures_NY_CAL_preprocessing_demo}c), we can see that the entire bi-objective space can be decomposed into four disjoint, highly-correlated linear relationships---a pattern observed consistently across the DIMACS dataset.
% %
% These four modes of correlation were detected by the RANSAC method (Alg.~\ref{alg:ransac_lines}).
% %%
% Importantly, each mode needs to be further subdivided into correlated clusters which are depicted in Fig.~\ref{fig:combined_figures_NY_CAL_preprocessing_demo}b,\ref{fig:combined_figures_NY_CAL_preprocessing_demo}d.

\ignore{ illustrates the outputs of the preprocessing framework applied to the DIMACS CAL instance. Fig.~\ref{fig:clusters_on_map}(a) shows the detected linear relationships by the RANSAC method (Alg. \ref{alg:ransac_lines}). As observed, the entire bi-objective space can be decomposed into four disjoint, highly-correlated linear relationships - a pattern observed consistently across the DIMACS dataset. The boundaries of the correlated clusters, aligned with these linear relationships, are delineated using Alg. \ref{alg:clusters_cca_detection} (Fig.~\ref{fig:clusters_on_map}(b)). The new query graph consists only of boundary vertices and their corresponding super-edges (not shown in the figure), as every vertex belongs to a non-trivial correlated cluster (Fig.~\ref{fig:clusters_on_map}(c)).
}

% \subsection{Lazy Edge Expansion Ablation Study}
% As demonstrated in Sec.~\ref{subsec:dimacs}, \Gtilde's branching factor is much larger than \G's which is why we suggested a lazy edge-expansion strategy (Sec.~\ref{sec:query_phase}). 
% %
% To this end, we compare (Fig.~\ref{fig:ablation_study}) 
% the query execution times of \eapex and \pegapex on \Gtilde, across various DIMACS instances for an approximation factor of $\beps=[0.01,0.01]$. \pegapex outperforms \eapex for almost all instances with the speed up in query times reaching above $5\times$. 

\subsection{\pegapex Query Runtimes}
We compare query running times of \pegapex and \apex, arguably the state-of-the-art algorithm for solving the approximate BOSP  problem (without preprocessing) on various DIMACS roadmaps.
We tested both on the highly-correlated DIMACS instances (Sec.~\ref{subsec:dimacs}) and then continue to generate a synthetic instances in which  we took the DIMACS \texttt{NY} instance and randomly sampled  edges costs  to form three linear \emph{non}-perfect correlations (Fig.~\ref{fig:synthetic_graph_ransac}).

For the highly-correlated DIMACS instances, other than a small number of outliers, \pegapex is always faster than \apex with maximal speed ups being well above $5\times$ (Fig.~\ref{fig:running_times_comparison}).
For the synthetic \texttt{NY}-based instance, we preprocessed the graph using  a fixed value of $\delta=0.05$ and 
four approximation factors $\beps~=~\bigl\{[0.001,0.001],[0.005,0.005],[0.01,0.01],[0.1,0.1]\bigr\}$. 
This combination of $\delta$ and $\beps$ keeps the number of vertices of~\Gtilde fixed while the average branching factor~$b(\Gtilde)$ increases as \beps-values decrease. 
Again, we compare query running times of \pegapex and \apex and can see (Fig.~\ref{fig:synthetic_graph_runtimes}) a dramatic speed up on most queries, reaching, in some instances, up to~$1000\times$.

\begin{figure}
    \centering
    \begin{subfigure}[t]{0.218\textwidth}  % Reduce width slightly if necessary
        \centering
        \includegraphics[width=1\textwidth]{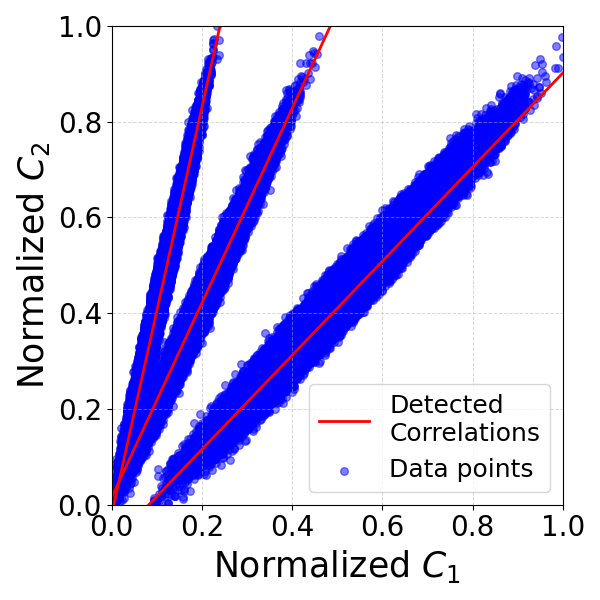}
        \caption{}         
        \label{fig:synthetic_graph_ransac}
    \end{subfigure}
    \hfill
    \begin{subfigure}[t]{0.25\textwidth}
        \centering
        \includegraphics[width=1\textwidth]{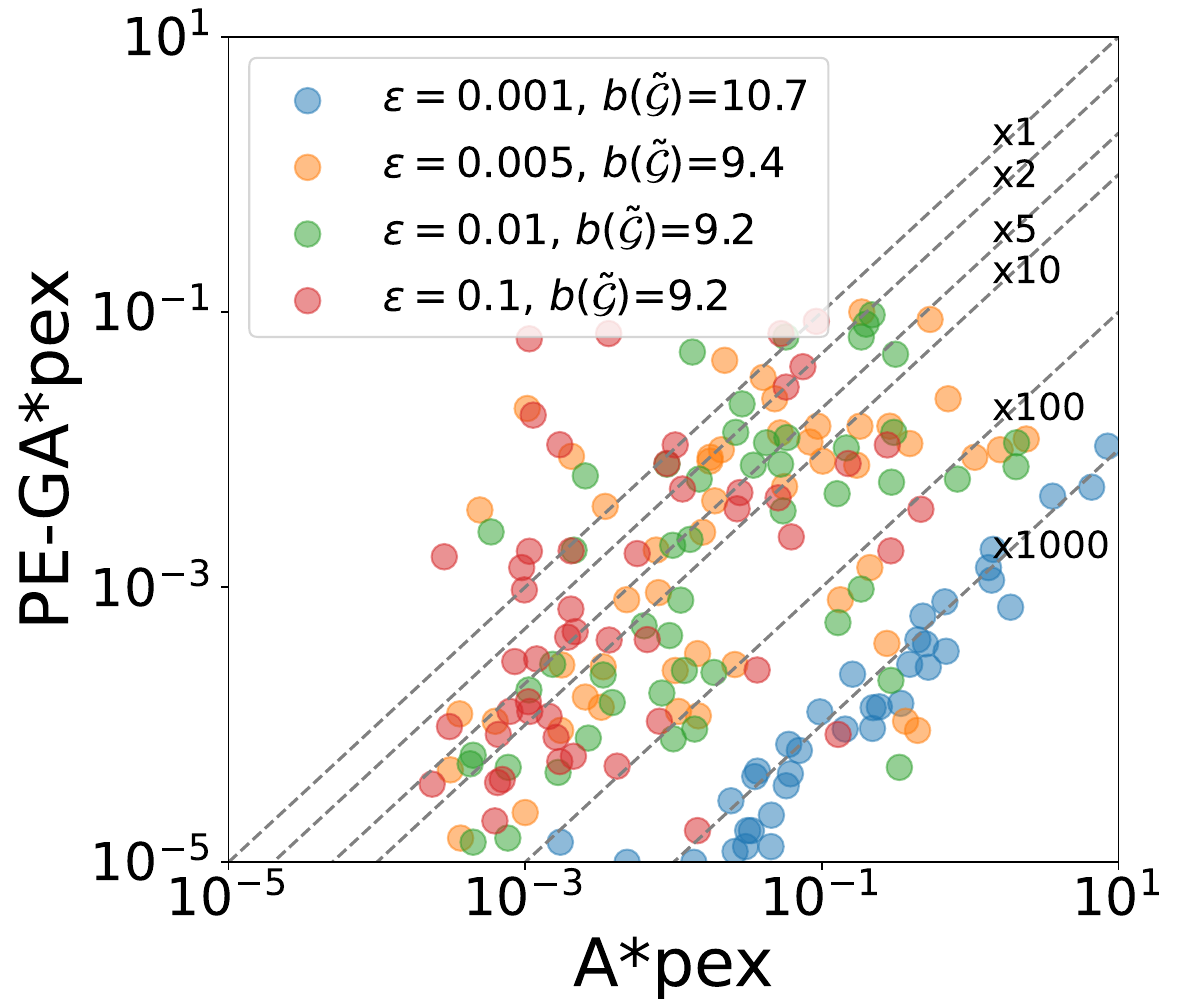}
        \caption{}         
        \label{fig:synthetic_graph_runtimes}
    \end{subfigure}
    \caption{{\textbf{(a)} 
            Edge cost plotted on the 2D objective costs (blue dots) and linear correlations computed by Alg.\ref{alg:ransac_lines} (red lines) for the synthetic bi-objective graph.
            \textbf{(b)} Running times (in seconds) of \apex and \pegapex on different queries in a synthetic bi-objective graph for varying \beps vectors.}}       
    \label{fig:combined_figures_synthetic_graph}
    \vspace{-4.5mm}
\end{figure}

\ignore{
\begin{figure*}
    \centering
    \begin{subfigure}[t]{0.49\textwidth}  % Reduce width slightly if necessary
        \centering
        \includegraphics[width=\textwidth]{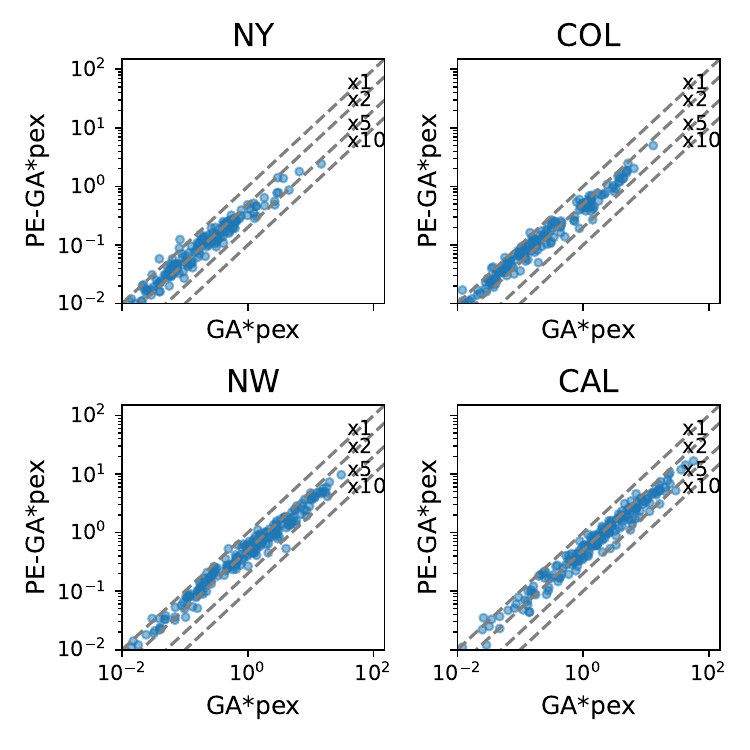}
        \caption{\textbf{Ablation Study}: evaluating the effectiveness of \pegapex compared to \eapex in mitigating the impact of an increased branching factor when searching over graph \Gtilde. The running times (in seconds) of both algorithms are presented for several DIMACS instances with $\boldsymbol{\varepsilon} = [0.01, 0.01]$.}         
        \label{fig:ablation_study}
    \end{subfigure}
    \hfill
    \begin{subfigure}[t]{0.49\textwidth}
        \centering
        \includegraphics[width=\textwidth]{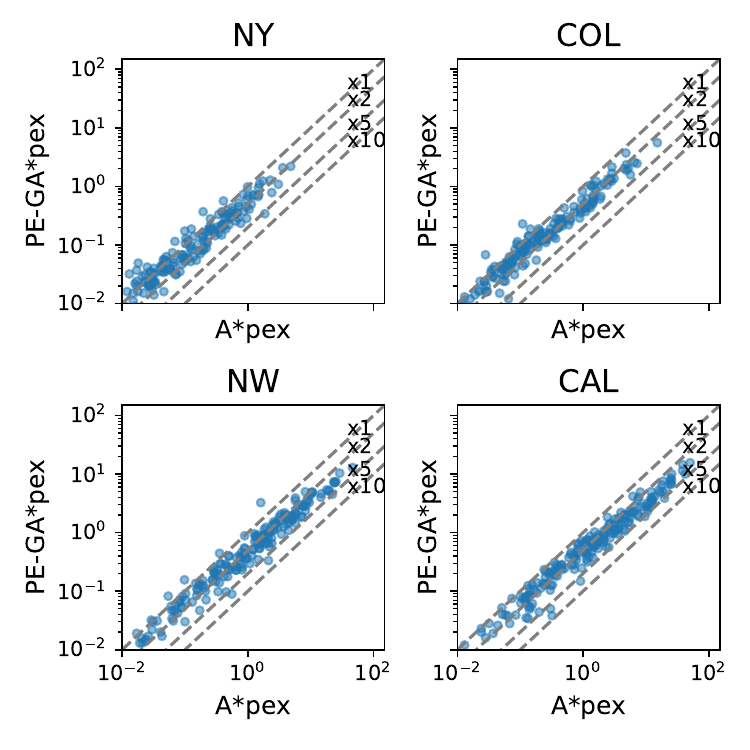}
        \caption{Running times (in seconds) of \apex and \pegapex in several DIMACS instances for $\boldsymbol{\varepsilon}=[0.01,0.01]$.}         
        \label{fig:running_times_comparison}
    \end{subfigure}
    
    \caption{Comparison of \pegapex and \apex in terms of running times and impact of the increased branching factor.}
    \label{fig:combined_figures}
\end{figure*}
}
\ignore{
\begin{figure}
            \centering
            \includegraphics[scale=0.7]{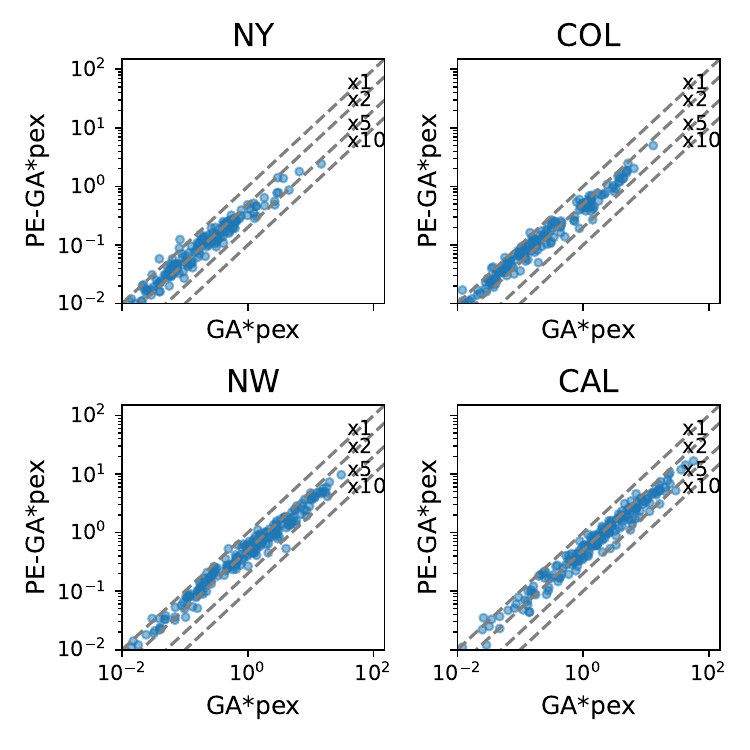}
            \caption{\textbf{Ablation Study}: evaluating the effectiveness of \pegapex compared to \eapex in mitigating the impact of an increased branching factor when searching over graph \Gtilde. The running times (in seconds) of both algorithms are presented for several DIMACS instances with $\boldsymbol{\varepsilon} = [0.01, 0.01]$.}         
            \label{fig:ablation_study}
\end{figure}
\begin{figure}
            \centering
            \includegraphics[scale=0.7]{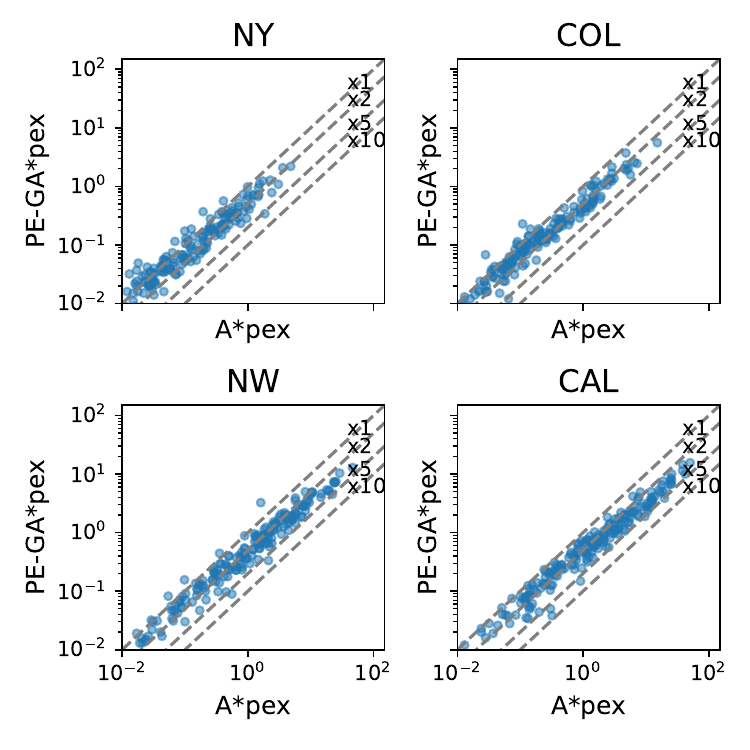}
            \caption{Running times (in seconds) of \apex and \pegapex in several DIMACS instances for $\boldsymbol{\varepsilon}=[0.01,0.01]$.}         
            \label{fig:running_times_comparison}
\end{figure}
}
\ignore{
\begin{figure*}[t]
            \centering
            \includegraphics[scale=0.6]{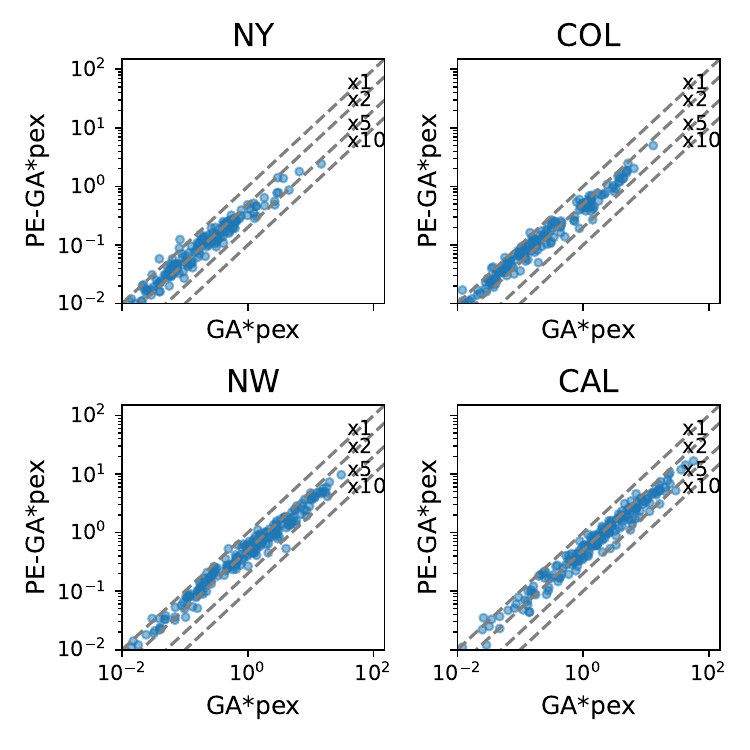}
            \caption{\textbf{Ablation Study}: evaluating the effectiveness of \pegapex compared to \eapex in mitigating the impact of an increased branching factor when searching over graph \Gtilde. The running times (in seconds) of both algorithms are presented for several DIMACS instances with $\boldsymbol{\varepsilon} = [0.01, 0.01]$.}         
            \label{fig:ablation_study}
\end{figure*}

\begin{figure*}[t]
            \centering
            \includegraphics[scale=0.6]{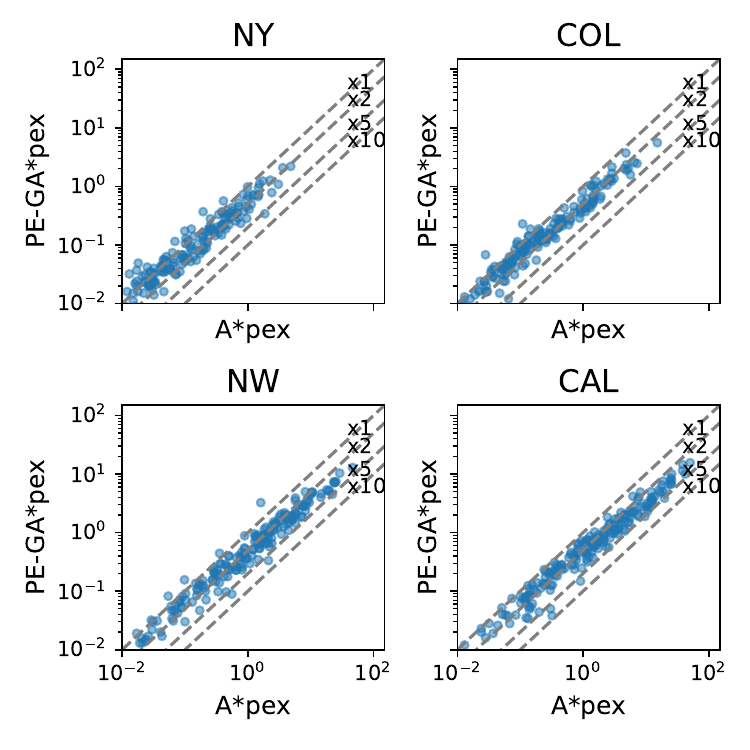}
            \caption{Running times (in seconds) of \apex and \pegapex in several DIMACS instances for $\boldsymbol{\varepsilon}=[0.01,0.01]$.}         
            \label{fig:running_times_comparison}
\end{figure*}
}

\section{Discussion and Future Work}
In this work we presented the first practical, systematic approach to exploit correlation between objectives in BOSP.
Our approach is based on a generalization to \apex that is of independent interest and an immediate question is what other problems can make use of this new algorithmic building block.
Our empirical evaluation on standard DIMACS benchmarks (Sec.~\ref{sec:eval}) indicate that costs of edges in instances of this dataset follow a nearly-perfect correlation (Fig.~\ref{fig:CAL_RANSAC_demo}). Thus, this dataset may be too synthetic to represent real-world data and better benchmarks are in need (a gap already identified by Salzman et al.~(\citeyear{salzman2023heuristic})).

As for future work, in our framework we introduce~$\delta$ to control which edges are considered to have the same correlation. This parameter is intimately related to the approximation factor \beps. Automatically choosing~$\delta$ according to a given value of \beps would reduce  the algorithm's parameters.

One other avenue for future work includes extending our framework to more than two objectives. This is highly challenging, as the number of correlated objectives may vary across different parts of the graph. 
Finally, it is extremely interesting to integrate our approach within the contraction hierarchies-based framework recently proposed by Zhang et al.~(\citeyear{zhang2023efficient}) for  exact BOSP problems. 

\section*{Acknowledgments}
This research was supported by Grant No. 2021643 from the United States-Israel Binational Science Foundation (BSF).

%\small
\bibliography{aaai25}

\appendix
\section{Appendix}
In this appendix, we present the detailed proofs of the key lemmas and theorems stated in the paper. For clarity and completeness, we restate each theoretical result before providing its proof.

\subsection{Proof of Lemma \ref{lemma:eps-bound}}
\label{apndx_lemma_eps_bound}
% \begin{lemma}
% Let $\AP=\langle \bold{A}, \pi \rangle$ be an $\beps$-bounded apex-path pair and 
% let $\ApEd=\langle \bold{EA},{e} \rangle$ be an outgoing $\beps$-bounded apex-edge pair connecting $v(\pi)$ to some vertex $v'$. 
% %
% If~$\AP'=~\langle \bold{A'},\pi' \rangle$ is the apex-path pair constructed by extending \AP by \ApEd, 
% then~$\AP'$ is $\beps$-bounded.
% \end{lemma}

\begingroup
  \renewcommand{\thetheorem}{4.1}
  \begin{lemma}
  \label{apndx_lemma_eps_bound_statement}
    Let $\AP=\langle \bold{A}, \pi \rangle$ be an $\beps$-bounded apex-path pair and 
    let $\ApEd=\langle \bold{EA},{e} \rangle$ be an outgoing $\beps$-bounded apex-edge pair connecting $v(\pi)$ to some vertex $v'$. 
    
    If~$\AP'=~\langle \bold{A'},\pi' \rangle$ is the apex-path pair constructed by extending \AP by \ApEd, 
    then~$\AP'$ is $\beps$-bounded.
  \end{lemma}
\endgroup
\addtocounter{theorem}{-1}

\begin{proof}
$\AP=\langle \bold{A}, \pi \rangle$ is \beps-bounded, thus:
\begin{equation}
    \label{eq:ap_eps_bounded}
    \mathbf{c}(\pi) \leq (1+\beps) \cdot \mathbf{A}.
\end{equation}

\noindent $\ApEd=\langle \bold{EA},{e} \rangle$ is \beps-bounded, thus:
\begin{equation}
    \label{eq:ae_eps_bounded}
    \mathbf{c}(e) \leq (1+\beps) \cdot \mathbf{EA}.
\end{equation}

\noindent The cost of path $\pi'$ is:
\begin{eqnarray}
    \begin{split}
    \mathbf{c}(\pi') &= \mathbf{c}(\pi) + \mathbf{c}(e) \\
    &\underbrace{\leq}_{\eqref{eq:ap_eps_bounded},\eqref{eq:ae_eps_bounded}} (1+\beps) \cdot \bold{A}+(1+\beps) \cdot \bold{EA} \\
    &=(1+\beps) \cdot (\bold{A} + \bold{EA}) \\
    &=(1+\beps) \cdot \mathbf{A'}. 
    \end{split}
\end{eqnarray}

\noindent Thus, by definition, $\AP'$ is \beps-bounded.
\end{proof}

\subsection{Proof of Theorem \ref{thm:gapex}}
\label{apndx_thm_gapex}

\begingroup
  \renewcommand{\thetheorem}{4.2}
  \begin{theorem}
    \label{apndx_thm_gapex_statement}
    Let $\hat{\mathcal{G}} = (\mathcal{V}, \mathcal{E},\mathbf{c}, \mathbf{c'})$    
    be a generalized graph of graph $\mathcal{G}=  (\mathcal{V}, \mathcal{E},\mathbf{c})$. 
    Let~$\vs,\vt \in~\mathcal{V}$ and recall that~$\Pi^*_{\bold{c}}$ denotes the Pareto-optimal set of paths between $v_s$ and~$v_t$ in $\G$.
    Set 
    $\beps := \underset{e\in \mathcal{E}}{\max} \left(\mathbf{c}(e) / \mathbf{c'}(e) - 1\right)$
    and
    let
    $\Pi^*_{\eapex}$ be the output of \eapex on $\hat{\mathcal{G}}$ when using an approximation factor $\beps$.
    Then, $\Pi^*_{\eapex} \preceq_{\beps} \Pi^*_{\bold{c}}$.
    Namely, running  \eapex on the generalized graph with approximation factor $\beps$ yields a Pareto-optimal solution set of paths between \vs and \vt that is an $\beps$-approximation of the Pareto-optimal solution set of paths between $v_s$ and~$v_t$ in \G.
  \end{theorem}
\endgroup
\addtocounter{theorem}{-1}

The proof of Thm.~\ref{thm:gapex} closely follows the optimality proof of \apex (Thm. 1 in~\cite{zhang2022pex}). As previously discussed, \apex and \eapex differ only in how they expand apex-path pairs. In the case of \eapex, the correctness argument relies on Lemma~\ref{lemma:eps-bound}, which establishes \beps-boundedness for each expanded apex-path pair.

For completeness, we reproduce the original proof of \apex from~\cite{zhang2022pex}, highlighting in blue the necessary modifications for the \eapex setting. The pseudo-code line numbers refer to \apex pseudo-code (Algorithm 2 in~\cite{zhang2022pex}).

We now present the details of Lemma 4 and Thm. 1 from~\cite{zhang2022pex}. Note that Lemmas 1–3 from the \apex paper extend directly to \eapex.

\subsubsection{Modified Proof of Lemma 4 (\apex paper)}
\label{lemma_4_apex_modified}
\begingroup
  \renewcommand{\thetheorem}{4.3}
  \begin{lemma}
    \label{lemma_4_apex_modified_statement}
    For any prefix $\pi_l = [s_1, s_2 \ldots s_l]$ of any solution $\pi = [s_1(=s_\text{start}), s_2\ldots s_L(=s_\text{goal})]$ with $1 \leq l \leq L$, there exists, when \textcolor{blue}{\eapex} terminates, (Case 1:) an expanded \apexnode $\AP$ (that is, one that reaches Line 9) that contains state $s_l$ and whose apex weakly dominates the $\mathbf{g}$-value of path $\pi_l$ or (Case 2:) an \apexnode $\AP$ in the solution set such that the $\mathbf{f}$-value of its representative path \beps-dominates the $\mathbf{f}$-value of path $\pi_l$.
  \end{lemma}
\endgroup
\addtocounter{theorem}{-1}

\begin{proof}
    The proof is by induction. The lemma holds for $l=1$ and any solution since apex-path pair $\AP=\langle \mathbf{0},[s_{\text{start}}] \rangle$ gets expanded and has the properties required for Case 1. Now assume that the lemma holds for some $l<L$ and any solution. We prove that it then also holds for $l+1$ and this solution.

    Assume that Case 1 holds for $l$ and consider both the apex-path pair \AP mentioned there and its potential child apex-path pair $\AP'$ created on Line 14 for $s'=s_{l+1}$. Apex-path pair $\AP'$ contains state $s_{l+1}$, and its apex weakly dominates the $\mathbf{g}$-value of path $\pi_{l+1}$, which implies that its $\mathbf{f}$-value weakly dominates the $\mathbf{f}$-value of path $\pi_{l+1}$. We distinguish several cases:
    \begin{enumerate}
        \item First, the condition on Line 20 holds for some apex-path pair in the solution set, namely, the truncated $\mathbf{f}$-value of the representative path of this apex-path pair \beps-dominates the truncated $\mathbf{f}$-value of apex-path pair $\AP'$. \textcolor{blue}{\eapex} replaces this apex-path pair with a new apex-path pair $\AP''$ in the solution set on Line 22. Apex-path pair $\AP''$ stays in the solution set but \textcolor{blue}{\eapex} might merge it several (more) times with other apex-path pairs on Line 29 before it terminates. The apex of apex-path pair $\AP''$ weakly dominates the $\mathbf{f}$-value of path $\pi_{l+1}$ (since this apex is the component-wise minimum of the $\mathbf{f}$-value of apex-path pair $\AP'$ and another apex and hence weakly dominates the $\mathbf{f}$-value of apex-path pair $\AP'$, which in turn weakly dominates the $\mathbf{f}$-value of path $\pi_{l+1}$) and merging it with other apex-path pairs does not change this property according to Lemma 3 (in \apex paper). This apex-path also remains \beps-bounded (which is due to the conditions on Lines 20 and 30, Lemma 1 (in \apex paper), the consistent heuristics and \textcolor{blue}{by Lemma \ref{lemma:eps-bound}}), the $\mathbf{f}$-value of its representative path always \beps-dominates the $\mathbf{f}$-value of itself, which equals its apex. Put together, the $\mathbf{f}$-value of its representative path \beps-dominates the $\mathbf{f}$-value of path $\pi_{l+1}$. Thus, the merged apex-path pair satisfies Case 2 for $l+1$.

        \item Second, the condition on Line 24 holds, namely, there exists a truncated $\mathbf{f}$-value in $G^T_{cl}(s(\AP'))$ that weakly dominates the truncated $\mathbf{f}$-value of apex-path pair $\AP'$. Then, an expanded apex-path pair $\AP''$ exists according to Lemma 2 (in \apex paper) that contains state $s_{l+1}$ and whose $\mathbf{f}$-value weakly dominates the $\mathbf{f}$-value of apex-path pair $\AP'$. Thus, its apex weakly dominates the apex of apex-path pair $\AP'$. Thus, apex-path pair $\AP''$ satisfies Case 1 for $l+1$ since the apex of apex-path pair $\AP'$ in turn weakly dominates the $\mathbf{g}$-value of path $\pi_{l+1}$.

        \item Otherwise, \textcolor{blue}{\eapex} executes Line 17 for apex-path pair $\AP'$, where the apex-path pair is inserted into \open, perhaps after having been merged with another apex-pair pair on Line 29. \textcolor{blue}{\eapex} might merge it several (more) times with other apex-path pairs on Line 29 before finally extracting it. Its apex weakly dominates the $\mathbf{g}$-value of path $\pi_{l+1}$ and merging it with other apex-path pairs does not change this property according to Lemma 3 (in \apex paper). Thus, if this apex-path pair is expanded, it satisfies Case 1 for $l+1$. If it is extracted but not expanded, the condition on Line 20 or Line 24 holds, and thus, as we have already proved, Case 1 or Case 2 holds.
    \end{enumerate}
    Assume that Case 2 holds for $l$ and consider the apex-path pair mentioned there. The $\mathbf{f}$-value of the representative path of this apex-path pair \beps-dominates the $\mathbf{f}$-value of path $\pi_l$. Since the heuristic function is consistent, the $\mathbf{f}$-value of path $\pi_l$ in turn weakly dominates the $\mathbf{f}$-value of path $\pi_{l+1}$. Thus, this apex-path pair satisfies Case 2 for $l+1$.
    \end{proof}
    
    \subsubsection{Modified Proof of Thm. 1 (\apex paper)}
    \label{thm_1_apex_modified}
    \begingroup
    \renewcommand{\thetheorem}{4.3}
      \begin{theorem}
        \label{thm_1_apex_modified_statement}
        For any solution $\pi$, there exists, when \textcolor{blue}{\eapex} terminates, an apex-path pair in the solution set whose representative path \beps-dominates $\pi$.
      \end{theorem}
    \endgroup
    \addtocounter{theorem}{-1} 
    
    \begin{proof}
        Lemma~\ref{lemma_4_apex_modified_statement} holds for prefix $\pi_L=\pi$ of any solution $\pi$. In case its Case 2 holds, the theorem holds by definition for path $\pi$ since the $\mathbf{f}$-value of solutions (including those of the representative path and path $\pi$) are equal to their costs. In case its Case 1 holds, consider the apex-path pair mentioned there. This apex-path pair contains the goal state, and \textcolor{blue}{\eapex} thus executed Line 11 for it, where the apex-path pair was inserted into the solution set, perhaps after having been merged with another apex-path pair on Line 29. The apex-path pair stays in the solution set but \textcolor{blue}{\eapex} might merge it several (more) times with other apex-path pairs on Line 29 before it terminates. The apex of the apex-path pair weakly dominates the $\mathbf{g}$-value of path $\pi$ according to Lemma~\ref{lemma_4_apex_modified_statement} and merging it with other apex-path pairs does not change this property according to Lemma 3 (in \apex paper). Since the apex-path pair also remains \beps-bounded (which is due to the conditions on Lines 20 and 30, Lemma 1 (in \apex paper), the consistent heuristics and \textcolor{blue}{by Lemma \ref{lemma:eps-bound}}), the $\mathbf{f}$-value of its representative path always \beps-dominates the $\mathbf{f}$-value of itself, which equals its apex. Put together, the $\mathbf{f}$-value of its representative path \beps-dominates the $\mathbf{g}$-value of path $\pi$. Thus, the theorem holds by definition for path $\pi$ since the $\mathbf{g}-$ and $\mathbf{f}$-values of solutions (including those of the representative path and path $\pi$) are equal to their costs.        
    \end{proof}

\subsection{Proof of Theorem \ref{thm:query}}
\label{apndx_thm_query}
\begingroup
    \renewcommand{\thetheorem}{5.1}
      \begin{theorem}
        \label{apndx_thm_query_statement}
        Let \vs and \vt be the start and target vertices, respectively, of a search query. Running \eapex on the generalized query graph yields an \beps-approximation of $\Pi^*$ in \G.
      \end{theorem}
    \endgroup
    \addtocounter{theorem}{-1} 

\noindent Recall that \G and \Gtilde differ only by the interior subgraphs of correlated clusters that do not include \vs nor \vt. While \G holds the original graph in these parts, \Gtilde substitutes the cluster's interior by super-edges connecting the boundary vertices. 
We now once again reproduce Lemma 4 from~\cite{zhang2022pex} while adapting it to the special structure of a generalized query graph. We highlight in blue the necessary modifications for the \eapex setting. The pseudo-code line numbers refer to \apex pseudo-code (Alg.~2 in~\cite{zhang2022pex}). Subsequently, Thm.~\ref{thm:gapex} holds and the proof is complete.

\subsubsection{Modified Proof of Lemma 4 (\apex)}
\begingroup
    \renewcommand{\thetheorem}{4.4}
      \begin{lemma}
        For any prefix $\pi_l = [s_1, \ldots, s_l]$ of any solution $\pi = [s_1(=s_\text{start}), s_2,\ldots, s_L(=s_\text{goal})]$ with $1 \leq l \leq L$ 
        \textcolor{blue}{and with $s_l \notin \{ V_\psi \setminus B(\psi) \}$ for any cluster $\psi$ (i.e., $s_l$ is either not in any cluster or it lies on a cluster's boundary)},         
        there exists, when \textcolor{blue}{\eapex} terminates, (Case 1:) an expanded \apexnode $\AP$ (that is, one that reaches Line 9) \textcolor{blue}{in \Gtilde} that contains state $s_l$ and whose apex weakly dominates the $\mathbf{g}$-value of path $\pi_l$ \textcolor{blue}{in \G} or (Case 2:) an \apexnode $\AP$ in the solution set \textcolor{blue}{of \eapex} such that the $\mathbf{f}$-value of its representative path \beps-dominates the $\mathbf{f}$-value of path $\pi_l$ \textcolor{blue}{in \G}.
      \end{lemma}
    \endgroup
    \addtocounter{theorem}{-1} 
    
    \begin{proof} 

    The proof is by induction. The lemma holds for $l=1$ and any solution since \apexnode $\AP = \langle \bm{0}, [s_\text{start}] \rangle$ gets expanded and has the properties required for Case 1. Now assume that the lemma holds for some $l < L$ and any solution. We prove that it then also holds for $l+1$ and this solution.

    \textcolor{blue}{
    However,  we need to  distinguish between the following instances:
    Instance \textbf{I1} in which $s_{l} \notin \{ V_\psi \setminus B(\psi) \}$ for any cluster $\psi$
    and
    Instance \textbf{I2} in which $s_{l} \in \{ V_\psi \setminus B(\psi)\}$ for some cluster $\psi$ (i.e., $s_{l+1}$ is a boundary node and its parent lies within some cluster $\psi$).
    In instance \textbf{I1}, we can use the induction hypothesis for $\pi_l = [s_1, \ldots, s_l]$.
    However, for  instance \textbf{I2}, we represent $\pi_{l+1}$ as a concatenation of two sub paths 
    $\pi_{l+1} = \pi_{\not\psi} \cdot \pi_{\psi}$.
    Here,        
    }
    \textcolor{blue}{
    \begin{itemize}
        \item $\pi_{\not\psi}: = [s_1, s_2, \ldots, s_{l'}]$ with 
        $s_1 = s_{start}$, 
        $s_{l'} \in B(\psi)$. 
        \item $\pi_{\psi}: = [s_{\psi,0},  \ldots, s_{\psi,k}]$ such that $s_{\psi,j} \in \{\psi \setminus B(\psi)\}~\forall j \in \{0, \ldots, k-1\}$
        and
        $s_{\psi,k} = s_{l+1}$.
        Namely, all vertices of~$\pi_{\psi}$ except the last one lie inside (i.e., not on the boundary of) cluster $\psi$
        and the last vertex $s_{l+1}$ lies on the boundary of cluster $\psi$.
    \end{itemize}    
    Note that the induction hypothesis holds for $\pi_{\not\psi}$.
    }

    Assume that Case 1 holds \textcolor{blue}{for the induction step. We distinguish between instance \textbf{I1} and instance \textbf{I2}}.
    
    \textcolor{blue}{If we are at instance \textbf{I1}, the induction step holds} for $l$ and consider both the \apexnode $\AP$ mentioned there and its potential child \apexnode $\AP'$ created on Line 14 for $s' = s_{l+1}$. \apexnode $\AP'$ contains state $s_{l+1}$, and its apex weakly dominates the $\mathbf{g}$-value of path $\pi_{l+1}$, which implies that its $\mathbf{f}$-value weakly dominates the $\mathbf{f}$-value of path $\pi_{l+1}$.

    \textcolor{blue}{If we are at instance \textbf{I2}, the induction step holds for $l'$,} 
    and consider both the \apexnode $\AP$ mentioned there and
    \textcolor{blue}{
    the child apex-path pairs generated by following the super-edges $\hat{\mathcal{E}}_{\psi,l',l+1}$. Due to Property~\ref{prp:super_edges_aprox}, there exists a specific super-edge $\hat{e} \in \hat{\mathcal{E}}_{\psi,l',l+1}$ that generates \apexnode $\AP'$ that contains state $s_{l+1}$ and its apex weakly dominates the $\mathbf{g}$-value of path \textcolor{blue}{$\pi_{l+1}$ in \G.}    
    }

    \noindent We distinguish several cases:
    \begin{enumerate}
        \item First, the condition on Line 20 holds for some \apexnode in the solution set \textcolor{blue}{of \eapex in \Gtilde}, namely, the truncated $\mathbf{f}$-value of the representative path of this \apexnode $\varepsilon$-dominates the truncated $\mathbf{f}$-value of \apexnode $\AP'$.
        \textcolor{blue}{\eapex} replaces this \apexnode with a new \apexnode $\AP''$ in the solution set on Line 22. 
        \apexnode $\AP''$ stays in the solution set but \textcolor{blue}{\eapex} might merge it several (more) times with other \apexnodes on Line 29 before it terminates. The apex of \apexnode $\AP''$ weakly dominates the $\mathbf{f}$-value of path \textcolor{blue}{$\pi_{l+1}$} (since this apex is the component-wise minimum of the $\mathbf{f}$-value of \apexnode $\AP'$ and another apex and hence weakly dominates the $\mathbf{f}$-value of \apexnode $\AP'$, which in turn weakly dominates the $\mathbf{f}$-value of path \textcolor{blue}{$\pi_{l+1}$}) and merging it with other \apexnodes does not change this property according to Lemma 3 (in \apex paper). Since the \apexnode also remains \beps-bounded (which is due to the conditions on Lines 20 and 30, Lemma 1 (in \apex paper), the consistent heuristics and \textcolor{blue}{by Lemma \ref{lemma:eps-bound}}), the $\mathbf{f}$-value of its representative path always \beps-dominates the $\mathbf{f}$-value of itself, which equals its apex. Put together, the $\mathbf{f}$-value of its representative path \beps-dominates the $\mathbf{f}$-value of path \textcolor{blue}{$\pi_{l+1}$}. 
        Thus, the merged \apexnode satisfies Case~2 for \textcolor{blue}{$l+1$}.
        \item Second, the condition on Line 24 holds, namely, there exists a truncated $\mathbf{f}$-value in $G^T_\text{cl}(s(\AP'))$ that weakly dominates the truncated $\mathbf{f}$-value of \apexnode $\AP'$. Then, an expanded \apexnode $\AP''$ exists according to Lemma 2 (in \apex paper) that contains state \textcolor{blue}{$s_{l+1}$} and whose $\mathbf{f}$-value weakly dominates the $\mathbf{f}$-value of \apexnode $\AP'$. Thus, its apex weakly dominates the apex of \apexnode $\AP'$. Thus, \apexnode $\AP''$ satisfies Case~1 for~\textcolor{blue}{$l+1$} since the apex of \apexnode $\AP'$ in turn weakly dominates the $\mathbf{g}$-value of path \textcolor{blue}{$\pi_{l+1}$}.
        \item Otherwise, \textcolor{blue}{\eapex} executes Line 17 for \apexnode $\AP'$, where the \apexnode is inserted into $\open$, perhaps after having been merged with another \apexnode on Line 29. \textcolor{blue}{\eapex} might merge it several (more) times with other \apexnodes on Line 29 before finally extracting it. Its apex weakly dominates the $\mathbf{g}$-value of path \textcolor{blue}{$\pi_{l+1}$} and merging it with other \apexnodes does not change this property according to Lemma 3 (in \apex paper). 
        Thus, if this \apexnode is expanded, it satisfies Case 1 for~\textcolor{blue}{$l+1$}. 
        If it is extracted but not expanded, the condition on Line~20 or Line~24 holds, and thus, as we have already proved, Case 1 or Case 2 holds.
    \end{enumerate}

    Assume that Case 2 holds for \textcolor{blue}{for the induction step. We distinguish between instance \textbf{I1} and instance \textbf{I2}.} 
    
    \textcolor{blue}{If we are at instance \textbf{I1}, the induction step holds for $l$} and consider the apex-path pair mentioned there. The $\mathbf{f}$-value of the representative path of this \apexnode \beps-dominates the $\mathbf{f}$-value of path $\pi_l$. Since the heuristic function is consistent, the $\mathbf{f}$-value of path~$\pi_l$ in turn weakly dominates the $\mathbf{f}$-value of path \textcolor{blue}{$\pi_{l+1}$}. Thus, this \apexnode satisfies Case 2 for \textcolor{blue}{$l+1$}. 

    \textcolor{blue}{If we are at instance \textbf{I2}, the induction step holds for $l'$,} and consider the apex-path pair mentioned there. The $\mathbf{f}$-value of the representative path of this \apexnode \beps-dominates the $\mathbf{f}$-value of path \textcolor{blue}{$\pi_{l'}$}. Since the heuristic function is consistent, the $\mathbf{f}$-value of path~\textcolor{blue}{$\pi_{l'}$} in turn weakly dominates the $\mathbf{f}$-value of path \textcolor{blue}{$\pi_{l}$} and path \textcolor{blue}{$\pi_{l+1}$}. Thus, this \apexnode satisfies Case 2 for \textcolor{blue}{$l+1$}. 
    \end{proof}

    \newpage
   %%%%%%%%%%%%%%%%%%%%%%%%%%%%%%%%%%%%%%%%%%%%%%%%%%%% 

\ignore{
\begingroup
    \renewcommand{\thetheorem}{4.4}
      \begin{lemma}
        For any prefix \textcolor{blue}{
         $\pi_l = [\underbrace{s_1, s_2, \ldots, s_a}_{\text{partial path out of $\psi$}},\underbrace{s_{a+1},\ldots, s_{b-1}}_{\text{partial path in $\psi$}},\underbrace{s_b, s_{b+1}, \ldots, s_l}_{\text{partial path out of $\psi$}}]$ in \G} of any solution \textcolor{blue}{$\pi = [s_1(=s_\text{start}), \ldots, s_l, \ldots,  s_L(=s_\text{goal})]$ in \G with $1 \leq l \leq L$, where 
        \begin{itemize}
            \item $s_a,s_b \in B(\psi)$ ($s_a$ and $s_b$ are boundary vertices of cluster $\psi$).
            \item $\vs,\vt \notin \mathcal{V}_\psi$ (\vs and \vt are not members of $\psi$).
            \item $\{  s_{a+1},\ldots,s_{b-1}\} \not\subset \Gtilde$ (the partial path connecting $s_a$ and $s_b$ is fully contained in $\psi$'s subgraph and thus removed from \Gtilde in the query phase).
            \item $\{ s_1,s_2,\ldots,s_a  \} \subset \G, \Gtilde$ (the prefix path going from $s_1$ to $s_a$ doesn't intersect any cluster, and as such, is identical in both graphs).
            \item $\{ s_b,s_{b+1},\ldots,s_L  \} \subset \G, \Gtilde$ (the suffix path going from $s_b$ to $s_L$ doesn't intersect any cluster, and as such, is identical in both graphs).
        \end{itemize}
        }
    there exists, when \textcolor{blue}{\eapex} terminates, (Case 1:) an expanded \apexnode $\AP$ (that is, one that reaches Line 9) \textcolor{blue}{in \Gtilde} that contains state $s_l$ and whose apex weakly dominates the $\mathbf{g}$-value of path $\pi_l$ \textcolor{blue}{in \G} or (Case 2:) an \apexnode $\AP$ in the solution set \textcolor{blue}{of \eapex} such that the $\mathbf{f}$-value of its representative path \beps-dominates the $\mathbf{f}$-value of path $\pi_l$ \textcolor{blue}{in \G}.
      \end{lemma}
    \endgroup
    \addtocounter{theorem}{-1} 
    
    \begin{proof} 
    The proof is by induction. 
    \textcolor{blue}
    %{Without loss of generality, we assume that \Gtilde contains a single cluster $\psi$. In the multi-cluster case, any path can be viewed as a sequence of segments, each crossing exactly one cluster, so the single-cluster argument applies to each segment independently.
    %
    {Recall that the partial path $[s_a,s_{a+1},\ldots,s_{b-1},s_b]$ exists in \G but \emph{not} in \Gtilde. In the query phase, the set of super-edges connecting $s_a$ to $s_b$ in $\psi$ ($\hat{\mathcal{E}}_{\psi,a,b}$) is added to \Gtilde. Therefore, for $l < a$, the lemma holds following Lemma~\ref{lemma_4_apex_modified_statement} as \G and \Gtilde are identical. We now prove that the lemma holds when \eapex expands an apex-path pair of $s_{l}=s_a$ to generate its successor $s_{l'}=s_b$ in \Gtilde via some super-edge. Recall that by Property~\ref{prp:super_edges_aprox}, for any partial path $\pi_{s_a \rightarrow s_b}$ of $\psi$ in \G, there exists a super-edge $\hat{e} \in \hat{\mathcal{E}}_{\psi,a,b}$ in \Gtilde where the $\mathbf{g}$-value of its apex weakly-dominates the $\mathbf{g}$-value of $\pi_{s_a \rightarrow s_b}$ (in the subgraph induced by $\psi$).}

    \textcolor{blue}{In \Gtilde, vertices $s_a$ and $s_b$ are connected directly by the set of super-edges $\hat{\mathcal{E}}_{\psi,a,b}$. In \G, however, the two vertices may not be adjacent; instead, there can be multiple partial paths within the subgraph of $\psi$ that lead from $s_a$ to $s_b$.}
    
    Assume that Case 1 holds for $l$ and consider both the \apexnode $\AP$ mentioned there and \textcolor{blue}{the several potential children apex-path pairs generated by following the super-edges $\hat{\mathcal{E}}_{\psi,a,b}$. Due to Property~\ref{prp:super_edges_aprox}, there exists a specific super-edge $\hat{e} \in \hat{\mathcal{E}}_{\psi,a,b}$ that generates \apexnode $\AP'$ that contains state $s_{l'} = s_b$ and its apex weakly dominates the $\mathbf{g}$-value of path \textcolor{blue}{$\pi_{l'}$ in \G}}, which implies that its $\mathbf{f}$-value weakly dominates the $\mathbf{f}$-value of path \textcolor{blue}{$\pi_{l'}$ in \G}. We distinguish several cases:
    \begin{enumerate}
        \item First, the condition on Line 20 holds for some \apexnode in the solution set \textcolor{blue}{of \eapex in \Gtilde}, namely, the truncated $\mathbf{f}$-value of the representative path of this \apexnode $\varepsilon$-dominates the truncated $\mathbf{f}$-value of \apexnode $\AP'$.
        \textcolor{blue}{\eapex} replaces this \apexnode with a new \apexnode $\AP''$ in the solution set on Line 22. 
        \Apexnode $\AP''$ stays in the solution set but \textcolor{blue}{\eapex} might merge it several (more) times with other \apexnodes on Line 29 before it terminates. The apex of \apexnode $\AP''$ weakly dominates the $\mathbf{f}$-value of path \textcolor{blue}{$\pi_{l'}$} (since this apex is the component-wise minimum of the $\mathbf{f}$-value of \apexnode $\AP'$ and another apex and hence weakly dominates the $\mathbf{f}$-value of \apexnode $\AP'$, which in turn weakly dominates the $\mathbf{f}$-value of path \textcolor{blue}{$\pi_{l'}$}) and merging it with other \apexnodes does not change this property according to Lemma 3 (in \apex paper). Since the \apexnode also remains \beps-bounded (which is due to the conditions on Lines 20 and 30, Lemma 1 (in \apex paper), the consistent heuristics and \textcolor{blue}{by Lemma \ref{lemma:eps-bound}}), the $\mathbf{f}$-value of its representative path always \beps-dominates the $\mathbf{f}$-value of itself, which equals its apex. Put together, the $\mathbf{f}$-value of its representative path \beps-dominates the $\mathbf{f}$-value of path \textcolor{blue}{$\pi_{l'}$}. 
        Thus, the merged \apexnode satisfies Case~2 for \textcolor{blue}{$l'$}.
        \item Second, the condition on Line 24 holds, namely, there exists a truncated $\mathbf{f}$-value in $G^T_\text{cl}(s(\AP'))$ that weakly dominates the truncated $\mathbf{f}$-value of \apexnode $\AP'$. Then, an expanded \apexnode $\AP''$ exists according to Lemma 2 (in \apex paper) that contains state \textcolor{blue}{$s_{l'}$} and whose $\mathbf{f}$-value weakly dominates the $\mathbf{f}$-value of \apexnode $\AP'$. Thus, its apex weakly dominates the apex of \apexnode $\AP'$. Thus, \apexnode $\AP''$ satisfies Case~1 for~\textcolor{blue}{$l'$} since the apex of \apexnode $\AP'$ in turn weakly dominates the $\mathbf{g}$-value of path \textcolor{blue}{$\pi_{l'}$}.
        \item Otherwise, \textcolor{blue}{\eapex} executes Line 17 for \apexnode $\AP'$, where the \apexnode is inserted into $\open$, perhaps after having been merged with another \apexnode on Line 29. \textcolor{blue}{\eapex} might merge it several (more) times with other \apexnodes on Line 29 before finally extracting it. Its apex weakly dominates the $\mathbf{g}$-value of path \textcolor{blue}{$\pi_{l'}$} and merging it with other \apexnodes does not change this property according to Lemma 3 (in \apex paper). 
        Thus, if this \apexnode is expanded, it satisfies Case 1 for~\textcolor{blue}{$l'$}. 
        If it is extracted but not expanded, the condition on Line~20 or Line~24 holds, and thus, as we have already proved, Case 1 or Case 2 holds.
    \end{enumerate}
    Assume that Case 2 holds for $l$ and consider the apex-path pair mentioned there. The $\mathbf{f}$-value of the representative path of this \apexnode \beps-dominates the $\mathbf{f}$-value of path $\pi_l$. Since the heuristic function is consistent, the $\mathbf{f}$-value of path~$\pi_l$ in turn weakly dominates the $\mathbf{f}$-value of path \textcolor{blue}{$\pi_{l'}$}. Thus, this \apexnode satisfies Case 2 for \textcolor{blue}{$l'$}. 

    \textcolor{blue}{Having established that the lemma holds for the case $l'=~b$, it immediately extends to all $l>b$, because beyond cluster $\psi$, the graphs \G and \Gtilde are identical.}
    \end{proof}
    }
\end{document}